\documentclass{article}

\usepackage{arxiv_preprint,times}
\usepackage{microtype}

\usepackage{amsmath,amsfonts,bm}

\def\eqref#1{equation~\ref{#1}}
\def\1{\bm{1}}

\DeclareMathAlphabet{\mathsfit}{\encodingdefault}{\sfdefault}{m}{sl}
\SetMathAlphabet{\mathsfit}{bold}{\encodingdefault}{\sfdefault}{bx}{n}

\usepackage{float}
\usepackage{amssymb}
\usepackage{array}
\usepackage{booktabs}
\usepackage{enumitem}
\usepackage{graphicx}
\usepackage{listings}
\usepackage{xcolor}
\usepackage{colortbl}
\definecolor{ccBlue}{HTML}{2A78D6}
\definecolor{ccOrange}{HTML}{EB6834}
\definecolor{ccGradeOne}{HTML}{86B6EF}
\definecolor{ccGradeZero}{HTML}{2A78D6}
\definecolor{ccNoVerdict}{HTML}{104281}
\definecolor{ccAuditRow}{HTML}{EDF4FC}
\usepackage{xspace}
\usepackage{tikz}
\usetikzlibrary{arrows.meta,backgrounds,fit,positioning,calc,shapes.misc,shapes.geometric}
\usepackage{hyperref}
\hypersetup{hidelinks}
\usepackage{url}

\usepackage{amsthm}

\AddToHook{env/table/begin}{\setlength{\belowcaptionskip}{4pt}}

\newtheorem{theorem}{Theorem}[section]
\newtheorem{lemma}[theorem]{Lemma}

\newcommand{\Fyan}{\textsc{Fyan}\xspace}

\newcommand{\Specifier}{\textsc{Specifier}\xspace}
\newcommand{\QED}{\textsc{QED}\xspace}
\newcommand{\Reasoner}{\textsc{Reasoner}\xspace}
\newcommand{\Planner}{\textsc{Planner}\xspace}
\newcommand{\Formalizer}{\textsc{Formalizer}\xspace}
\newcommand{\Archon}{\textsc{Archon}\xspace}
\newcommand{\Curator}{\textsc{Curator}\xspace}
\newcommand{\Auditor}{\textsc{Auditor}\xspace}
\newcommand{\Orchestrator}{\textsc{Orchestrator}\xspace}
\newcommand{\code}[1]{\texttt{#1}}
\newcommand{\link}[1]{Link~#1}

\usepackage{amsthm}
\usepackage{multirow}
\usepackage{pifont}
\makeatletter
\@ifundefined{theorem}{\newtheorem{theorem}{Theorem}}{}
\@ifundefined{lemma}{\newtheorem{lemma}{Lemma}}{}
\@ifundefined{definition}{}{}
\makeatother

\newcommand{\Ntcs}{143}

\newcommand{\ftfpFyan}{86}          % 86/143
\newcommand{\ftfpFyanPct}{60.1}
\newcommand{\ftfpFyanCI}{52.0--67.8} % Wilson 95%
\newcommand{\ftfpCtl}{69}           % 69/143
\newcommand{\ftfpCtlPct}{48.3}
\newcommand{\ftfpCtlCI}{40.2--56.4}
\newcommand{\ftfpGain}{11.9}        % percentage points
\newcommand{\ftfpFisher}{0.057}     % unpaired two-sided Fisher exact test

\newcommand{\tnpFyan}{0.851}        % 143 items, sum 121.74 (3 unscored count as 0)

\newcommand{\tnpCtl}{0.501}         % 71.59/143: 109 scored items (mean 0.657), 34 without a scored proof count as 0 (re-confirmed 2026-09-25)
\newcommand{\tnpCtlUnscored}{34}

\newcommand{\ntftFyan}{8}           % faithfulness-judge run (September 18), author-confirmed
\newcommand{\ntftFyanPct}{5.6}
\newcommand{\ntftCtl}{5}
\newcommand{\ntftCtlPct}{3.5}
\newcommand{\ntftFisher}{0.57}      % two-sided Fisher, 8/143 vs 5/143
\newcommand{\genRefs}{91}           % references declaring structure/inductive/class
\newcommand{\defRefs}{52}           % references with definitions only
\newcommand{\genFisher}{0.0056}     % one-sided: all 5 control successes among the 52
\newcommand{\pubBest}{11.2}         % best system in FormalTCS v3 (Pass@8)

\newcommand{\pnsPass}{216}
\newcommand{\pnsElig}{343}          % items whose statement elaborates under Lean 4.32.2
\newcommand{\pnsAll}{371}
\newcommand{\pnsPct}{63.0}          % 216/343
\newcommand{\pnsCI}{57.7--67.9}     % Wilson 95%
\newcommand{\odeUnits}{23}
\newcommand{\odeDone}{20}
\newcommand{\odeLines}{9{,}355}
\newcommand{\odeThm}{34}
\newcommand{\odeLem}{118}

\newcommand{\ccTN}{300}
\newcommand{\ccTAuditAcc}{0.807}  \newcommand{\ccTAuditAccCI}{0.760--0.850}
\newcommand{\ccTAuditP}{0.771}      \newcommand{\ccTAuditR}{0.873}   \newcommand{\ccTAuditFone}{0.819}
\newcommand{\ccTDirAcc}{0.883}    \newcommand{\ccTDirAccCI}{0.847--0.917}
\newcommand{\ccTDirP}{0.939}        \newcommand{\ccTDirR}{0.820}     \newcommand{\ccTDirFone}{0.875}
\newcommand{\ccTMcNemar}{0.007}
\newcommand{\ccTAuditFP}{39}  \newcommand{\ccTAuditFN}{19}
\newcommand{\ccTDirFP}{8}    \newcommand{\ccTDirFN}{27}
\newcommand{\ccTPolicy}{18}    % grade-1 false alarms
\newcommand{\ccTGradeOne}{51}  \newcommand{\ccTGradeOneInc}{33}  % all grade-1 statements; of those labelled inconsistent
\newcommand{\ccTLegal}{18}    % of those, legal mutations or argued reformulations (author re-check: all 18)
\newcommand{\ccTMechFP}{21}    % false alarms from audit errors (strict, unpaired, source, no verdict) = 39 - 18
\newcommand{\ccTNoVerdict}{10}
\newcommand{\ccAdjN}{300}  \newcommand{\ccAdjRelabel}{56}  \newcommand{\ccAdjSame}{94}
\newcommand{\ccAdjEqNotSame}{36}  \newcommand{\ccAdjUnstated}{12}  \newcommand{\ccAdjNotEq}{8}
\newcommand{\ccAdjAuditR}{0.777}  \newcommand{\ccAdjAuditP}{0.941}  \newcommand{\ccAdjAuditAcc}{0.813}  \newcommand{\ccAdjAuditFone}{0.851}
\newcommand{\ccAdjAuditTP}{160}  \newcommand{\ccAdjAuditFN}{46}  \newcommand{\ccAdjAuditTN}{84}  \newcommand{\ccAdjAuditFP}{10}
\newcommand{\ccAdjDirR}{0.636}  \newcommand{\ccAdjDirP}{1.000}  \newcommand{\ccAdjDirAcc}{0.750}  \newcommand{\ccAdjDirFone}{0.777}
\newcommand{\ccTAuditFPgZero}{18}  \newcommand{\ccTAuditFPnv}{3}  % audit false alarms at grade 0 / without a valid record
    
\newcommand{\ccAdjInc}{206}  \newcommand{\ccAdjCon}{94}
\newcommand{\ccVerFPNon}{29}  \newcommand{\ccVerFPSame}{10}  % the 39 audit rejections: not the text's statement / same
\newcommand{\ccAdjAuditAccCI}{0.767--0.853}  \newcommand{\ccAdjDirAccCI}{0.700--0.797}
\newcommand{\ccAdjAuditRCI}{0.717--0.833}  \newcommand{\ccAdjDirRCI}{0.571--0.702}
\newcommand{\ccAdjGradeThree}{6}    % grade-3 statements (all verified consistent); precision at threshold 3 (recall 1.000)
\newcommand{\ccAdjDirAcceptsRelabel}{48}  \newcommand{\ccAdjAuditRejectsRelabel}{29}  % of the 56 relabelled statements
\newcommand{\ccAdjMcAudit}{43}  \newcommand{\ccAdjMcDir}{24}  \newcommand{\ccAdjMcP}{0.027}
\newcommand{\ccVerGOneNon}{13}  \newcommand{\ccVerGOneArg}{9}  \newcommand{\ccVerGOneOther}{4}  \newcommand{\ccVerGOneSame}{5}
\newcommand{\ccVerGZeroNon}{14}  \newcommand{\ccVerGZeroSame}{4}  \newcommand{\ccVerNvNon}{2}  \newcommand{\ccVerNvSame}{1}  \newcommand{\ccVerRestSame}{5}
\newcommand{\ccAdjDirTP}{131}  \newcommand{\ccAdjDirFN}{75}  \newcommand{\ccAdjDirTN}{94}  \newcommand{\ccAdjDirFP}{0}
\newcommand{\ccTNoVerdictInc}{7}  \newcommand{\ccTNoVerdictCon}{3}  % no-verdict runs on expert-inconsistent / -consistent items
\newcommand{\ccTExclR}{0.867}  \newcommand{\ccTExclP}{0.775}
\newcommand{\ccTOpenTP}{124}  \newcommand{\ccTOpenR}{0.827}  \newcommand{\ccTOpenAcc}{0.793}
\newcommand{\ccTOnlyAudit}{22}  % items right only for the audit
\newcommand{\ccTOnlyDir}{45}    % items right only for the direct judge
\newcommand{\ccTStruct}{6}     % of those, structural errors (scope, quantifier, interval)
\newcommand{\ccTAuditCost}{0.15}  \newcommand{\ccTDirCost}{0.007}   % CNY per item
\newcommand{\ccTThrOneTP}{98}  \newcommand{\ccTThrOneR}{0.653}  \newcommand{\ccTThrOneP}{0.824}

\title{\Fyan: A Human--AI Harness with Semantic Auditing for Document-Level Formalization}
\author{%
\textbf{Wei Zhao}$^{1}$\thanks{Equal contribution.} \quad
\textbf{Yangshuo Zou}$^{7}$\footnotemark[1] \quad
\textbf{Chengxiang Ding}$^{3}$ \quad
\textbf{Yifan Wu}$^{3}$\\
\textbf{Xuchuan Wang}$^{4}$ \quad
\textbf{Zimu Mao}$^{5}$ \quad
\textbf{Lei Zhang}$^{1,2}$\thanks{Corresponding authors.} \quad
\textbf{Tao Luo}$^{1,2,6}$\footnotemark[2]\\[0.75em]
{\small $^1$School of Mathematical Sciences, Shanghai Jiao Tong University, Shanghai 200240, China}\\
{\small $^2$Institute of Natural Sciences, MOE-LSC, Shanghai Jiao Tong University, Shanghai 200240, China}\\
{\small $^3$Zhiyuan College, Shanghai Jiao Tong University, Shanghai 200240, China}\\
{\small $^4$School of Biomedical Engineering, Shanghai Jiao Tong University, Shanghai 200240, China}\\
{\small $^5$School of Materials Science and Engineering, Shanghai Jiao Tong University, Shanghai 200240, China}\\
{\small $^6$CMA-Shanghai, Shanghai Jiao Tong University, Shanghai 200240, China}\\
{\small $^7$University of California, Berkeley, CA 94720, USA}%
}
\date{}
\begin{document}

\maketitle

\begin{abstract}
We present \Fyan, a human–AI harness for document-level mathematical formalization. Rather than treating theorems in isolation, \Fyan coordinates an end-to-end workflow spanning specification, proof planning, logical review, Lean proof construction, knowledge curation, and validation, with support for independent supervision and human guidance. 
A central component is evidence-grounded semantic auditing, which assesses whether formal statements faithfully preserve their informal specifications. 
A language model constructs structured evidence over local correspondences, omissions, scope, and logical relations, while a deterministic validator checks this evidence and produces reproducible judgments. 
When a substantive but admissible deviation is accepted, \Fyan requires an explicit proof-transfer obligation connecting the formal statement back to a source-facing interpretation. 
With the same model (DeepSeek-V4.1-Flash) in every stage, \Fyan proves 86 of 143 FormalTCS theorems under a strict Lean check, against 69 for a general agent harness, and raises the natural-language proof score from 0.501 to 0.851. 
On ConsistencyCheck, its semantic audit catches more inconsistent statements than a direct LLM judge, both on labels verified against the source (recall 0.777 vs. 0.636) and on the original labels (0.873 vs. 0.820), and localizes each mismatch it reports to a specific hypothesis, conclusion, or scope. 
\Fyan also built ODENumLib, a 9,355-line Lean library for the numerical analysis of ordinary differential equations.
% On real misformalizations, the measured v1
% slot rubric attains higher error recall than a direct judge (0.549 versus
% 0.324 on the \code{pairs} set), at lower specificity and greater cost. The
% relation-rubric and system-level evaluations remain pending and are not claimed
% as completed results.
% ---- new ----
\iffalse
We propose a four-level grade for semantic fidelity, computed in code from
local comparisons of the statement with its source, that separates a
statement aligned with the text from one whose proof merely transfers to it.
With DeepSeek-V4.1-Flash in every stage, \Fyan proves \ftfpFyan{} of \Ntcs{}
FormalTCS theorems under a strict Lean check against \ftfpCtl{} for a general
agent harness, and builds ODENumLib, a \odeLines-line Lean library of
numerical analysis. On ConsistencyCheck the audit finds more inconsistent
statements than a direct judge (recall \ccTAuditR{} versus \ccTDirR), and its
grades expose where expert consistency labels are themselves unsettled.
\fi
\end{abstract}

\section{Introduction}
\label{sec:intro}

Formal proof assistants such as Lean provide machine-checkable guarantees for
mathematical arguments, but formalization still requires substantial human
effort to translate informal statements, identify intermediate lemmas, locate
relevant library results, and construct and repair proofs. Recent advances in
large language models (LLMs) are changing this workflow. LLM-based systems can
generate mathematical arguments, translate informal statements into formal
specifications, construct proof-assistant code, and refine proofs using
compiler feedback
\citep{wu2022autoformalization,yang2023leandojo,jiang2023dsp,
xin2025deepseekproverv15,an2026qed}. As these capabilities improve, automated
formalization is expanding from isolated theorems toward research papers,
textbooks, and other collections of interdependent mathematical
results~\citep{mathinc2025gauss,achim2025aristotle,zhang2026leanmarathon,
gloeckle2026textbook,rammal2026autoformbot,wang2026m2f}.

Scaling from a single theorem to a mathematical document changes the problem
substantially. A system must identify the theorem units in the source, recover
their dependencies, determine a valid formalization order, reuse previously
established results, and maintain a coherent formal project. These requirements
naturally motivate agentic workflows in which different components handle
specification, reasoning, planning, proving, feedback, and reuse. Such
workflows extend proof-search capability, but they also make intermediate
formalization decisions increasingly consequential: once a generated theorem
is accepted as a dependency, later results may build on it automatically.

Semantic faithfulness is a central concern. A proof assistant checks
the theorem, not whether it faithfully represents the source. A
formal statement may compile and admit a proof while omitting a
hypothesis, changing a quantified domain, or replacing a source object with a
different formal construction. At the document level, such mismatches can propagate
through downstream dependencies while the project remains formally verified.
We therefore distinguish \emph{proof capability}, whether a proof can be found;
\emph{semantic faithfulness}, whether the formal statement preserves the source
meaning; and \emph{formal correctness}, whether the artifacts satisfy
machine-checkable requirements.

Existing work evaluates informal--formal alignment through bidirectional
provability, learned or prompted judgments, structural comparisons, expert
references, and targeted consistency checks
\citep{zhang2026beyondcompilation,liu2025beq,han2026shadowbench,
ammanamanchi2026faults}. These methods provide useful theorem-level signals,
but formalizing an entire document also requires turning this evidence
into document-level decisions: whether a generated statement may advance,
whether a deviation from the source is justified, and what evidence is
required before dependent theorems can safely build on it or the project can
be considered complete.

We introduce \Fyan, a human--AI harness for document-level mathematical
formalization. \Fyan decomposes a source into dependency-aware theorem units,
formalizes them in dependency order, and allows later units to reuse proved
upstream results. It places a semantic audit between statement generation and
proof search, while allowing models and human experts to guide the process
without directly deciding whether a theorem may advance. Our contributions are:

\begin{enumerate}[nosep,leftmargin=*]

  \item \textbf{A dependency-aware workflow for document-level formalization.}
        \Fyan organizes theorem units in a document-level dependency DAG and
        formalizes them in topological order, so proved upstream modules can be
        reused by downstream theorems (Figure~\ref{fig:pipeline}). Components
        exchange persistent project artifacts rather than relying on transient
        agent context, while the \Orchestrator{} tracks dependencies and project
        progress. Users and domain experts can inspect the project and provide
        guidance without directly editing protected statements or forcing stage
        transitions.

  \item \textbf{An evidence-grounded semantic audit with one standard of equivalence.}
        Before proof search proceeds, \Fyan compares the candidate Lean
        statement with its source and records local correspondences, omissions,
        and semantic differences (Figure~\ref{fig:semantic}). An LLM proposes
        the comparison evidence, while deterministic code checks coverage,
        scope, polarity, and admissible transfer directions and computes the
        grade. Only statements whose agreement with the source is evident are
        accepted directly; the rest return for repair or, if admitted, carry a
        bridge obligation that must be discharged in Lean.
\iffalse
  \item \textbf{Deterministic gates for stage advancement and completion.}
        A theorem advances only when the required artifacts and checks are
        current; an agent cannot advance the pipeline by simply reporting
        success. Final completion requires a successful build, no admissions,
        discharged bridge obligations, intact protected signatures, current
        certifications, and an approved transitive axiom audit.
\fi
 \item \textbf{Empirical evaluation across proving, faithfulness, and document-level formalization.}
With the model, toolchain, and graders fixed, \Fyan{} proves \ftfpFyan{} of
\Ntcs{} FormalTCS theorems under a strict Lean check against \ftfpCtl{} for a
general agent harness and raises the natural-language proof score from
\tnpCtl{} to \tnpFyan{}. Its
semantic audit catches more inconsistent statements on ConsistencyCheck than a
direct LLM judge (recall \ccAdjAuditR{} vs.\ \ccAdjDirR{} on verified labels and
\ccTAuditR{} vs.\ \ccTDirR{} on the original labels). \Fyan{} also built
ODENumLib, \odeDone{} dependency-linked theorem units and \odeLines{} lines of
Lean, including error bounds for $\theta$-methods and the consistency
conditions of linear multistep methods.
\end{enumerate}

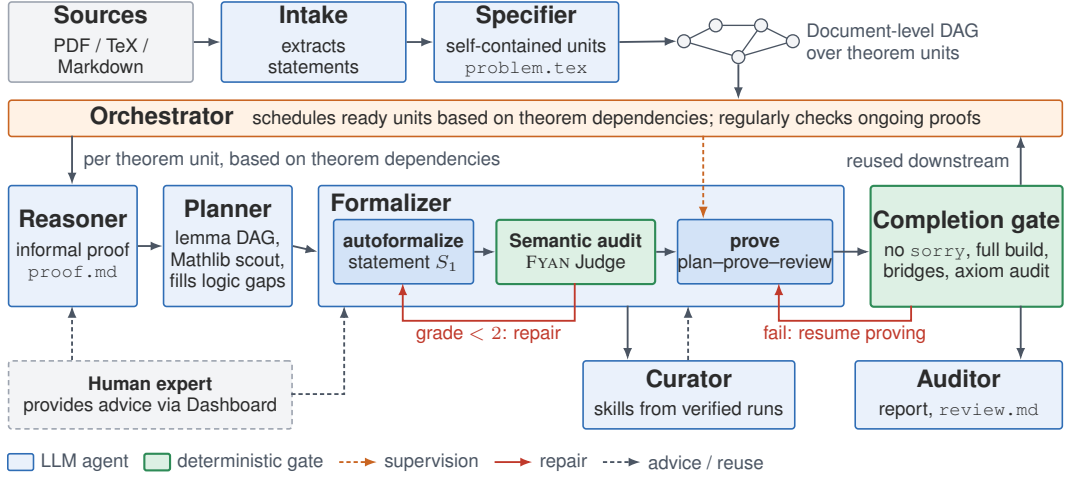
\begin{figure}[t]
\centering
\definecolor{fyBlue}{HTML}{2F6DB5}
\definecolor{fyBlueFill}{HTML}{EAF1FB}
\definecolor{fyBlueStep}{HTML}{D3E3F7}
\definecolor{fyGreen}{HTML}{2E8B57}
\definecolor{fyGreenFill}{HTML}{DDF0E4}
\definecolor{fyRed}{HTML}{C0392B}
\definecolor{fyOrange}{HTML}{C96522}
\definecolor{fyOrangeFill}{HTML}{FFF2E5}
\definecolor{fyInk}{HTML}{2B2B2B}
\definecolor{fyInkSoft}{HTML}{4B5563}
\definecolor{fyInkMute}{HTML}{6B7280}
\definecolor{fyLine}{HTML}{9CA3AF}
\definecolor{fyGray}{HTML}{F3F4F6}
\begin{tikzpicture}[x=1mm,y=1mm,
  font=\sffamily\scriptsize, text=fyInk,
  box/.style={rounded corners=1.2pt, line width=0.6pt, align=center,
    inner sep=0pt, anchor=north west},
  model/.style={box, draw=fyBlue, fill=fyBlueFill},
  gate/.style={box, draw=fyGreen, fill=fyGreenFill, line width=0.8pt},
  neutral/.style={box, draw=fyLine, fill=fyGray},
  substep/.style={box, draw=fyBlue, fill=fyBlueStep},
  arr/.style={-{Latex[length=1.6mm]}, line width=0.7pt, draw=fyInkSoft},
  rej/.style={-{Latex[length=1.6mm]}, line width=0.7pt, draw=fyRed},
  adv/.style={-{Latex[length=1.6mm]}, line width=0.7pt, draw=fyInkSoft,
    dash pattern=on 1.6pt off 1.1pt},
  sup/.style={-{Latex[length=1.6mm]}, line width=0.7pt, draw=fyOrange,
    dash pattern=on 1.6pt off 1.1pt},
  lbl/.style={font=\sffamily\scriptsize, text=fyInkSoft, inner sep=0pt},
]
\def\fyT#1{{\footnotesize\bfseries\strut #1}}% local helpers (scoped to this figure)
\def\fyS#1{{\bfseries\strut #1}}
\def\fyF#1{{\ttfamily #1}}
% ===================== document level (once per document) =====================
\node[neutral, minimum width=24.4mm, minimum height=10.6mm] (src) at (0,0)
  {\fyT{Sources}\\[0.4mm]\strut PDF / TeX /\\\strut Markdown};
\node[model, minimum width=24.4mm, minimum height=10.6mm] (intake) at (28.15,0)
  {\fyT{Intake}\\[0.4mm]\strut extracts\\\strut statements};
\node[model, minimum width=24.4mm, minimum height=10.6mm] (spec) at (56.3,0)
  {\fyT{Specifier}\\[0.4mm]\strut self-contained units\\\strut \fyF{problem.tex}};
\begin{scope}[shift={(96.65,-5.3)}]
  \coordinate (ga) at (-7.2,0);    \coordinate (gb) at (-3,2.1);
  \coordinate (gc) at (0,-2.1);    \coordinate (gd) at (3,2.1);
  \coordinate (ge) at (7.2,0);
  \draw[fyInkSoft, line width=0.5pt] (ga)--(gb) (ga)--(gc) (gb)--(gd)
    (gc)--(gd) (gc)--(ge) (gd)--(ge);
  \foreach \nd in {ga,gb,gc,gd,ge}
    \filldraw[draw=fyInkSoft, fill=fyGray, line width=0.5pt] (\nd) circle (0.85mm);
\end{scope}
\node[lbl, anchor=west, align=left] at (105.6,-5.3)
  {Document-level DAG\\ over theorem units};
\draw[arr] (src.east) -- (intake.west);
\draw[arr] (intake.east) -- (spec.west);
\draw[arr] (spec.east) -- ($(ga)+(-0.95,0)$);
% ===================== Orchestrator =====================
\node[box, draw=fyOrange, fill=fyOrangeFill, minimum width=139.4mm,
  minimum height=4.8mm] (orch) at (0,-13)
  {\fyT{Orchestrator}\quad schedules ready units based on theorem dependencies;
   regularly checks ongoing proofs};
\draw[arr] ($(gc)+(0,-0.95)$) -- (96.65,-13);
% ===================== per theorem unit =====================
\node[model, minimum width=17mm, minimum height=16mm] (rsn) at (0,-24.3)
  {\fyT{Reasoner}\\[0.4mm]\strut informal proof\\\strut \fyF{proof.md}};
\node[model, minimum width=17mm, minimum height=16mm] (pln) at (20.5,-24.3)
  {\fyT{Planner}\\[0.2mm]\strut lemma DAG,\\\strut Mathlib scout,\\\strut fills logic gaps};
% Formalizer: one component, two steps, with the audit gate between them
\node[model, minimum width=69.5mm, minimum height=16mm] (fml) at (41,-24.3) {};
\node[anchor=north west, inner sep=0pt] at (42.6,-25.2) {\fyT{Formalizer}};
\node[substep, minimum width=18.5mm, minimum height=8.5mm] (af) at (43,-28.8)
  {\fyS{autoformalize}\\ statement $S_1$};
\node[gate, minimum width=21mm, minimum height=8.5mm] (sa) at (64.5,-28.8)
  {\fyS{Semantic audit}\\ {\rmfamily\scshape Fyan} Judge};
\node[substep, minimum width=20.5mm, minimum height=8.5mm] (pr) at (88.5,-28.8)
  {\fyS{prove}\\ plan--prove--review};
\node[gate, minimum width=25.4mm, minimum height=16mm] (cg) at (114,-24.3)
  {\fyT{Completion gate}\\[0.4mm]\strut no \fyF{sorry}, full build,\\\strut bridges, axiom audit};
\draw[arr] (rsn) -- (pln);
\draw[arr] (pln.east) -- (fml.west|-af.west);
\draw[arr] (af) -- (sa);
\draw[arr] (sa) -- (pr);
\draw[arr] (pr.east) -- (cg.west|-pr.east);
% dispatch from the Orchestrator
\draw[arr] (8.5,-17.8) -- (8.5,-24.3);
\node[lbl, anchor=west] at (10,-21.05) {per theorem unit, based on theorem dependencies};
% scheduled supervision (advice only)
\draw[sup] (92,-17.8) -- (92,-28.8);
% verified module is reused downstream
\draw[arr] (134,-24.3) -- (134,-17.8);
\node[lbl, anchor=east, align=right] at (132.6,-21.05)
  {reused downstream};
% repair loops
\draw[rej] (75,-37.3) -- (75,-42.3) -- (52.25,-42.3) -- (52.25,-37.3);
\node[lbl, text=fyRed, anchor=north] at (63.6,-42.8) {grade $<2$: repair};
\draw[rej] (119.5,-40.3) -- (119.5,-42.3) -- (102,-42.3) -- (102,-37.3);
\node[lbl, text=fyRed, anchor=north] at (110.75,-42.8) {fail: resume proving};
% ===================== advice and reuse =====================
\node[neutral, dash pattern=on 2pt off 1.4pt,
  minimum width=37.5mm, minimum height=9mm] (hum) at (0,-47.5)
  {\fyS{Human expert}\\ provides advice via Dashboard};
\draw[adv] (8.5,-47.5) -- (8.5,-40.3);
\draw[adv] (37.5,-52) -- (44.5,-52) -- (44.5,-40.3);
\node[model, minimum width=28mm, minimum height=9mm] (cur) at (76,-47.5)
  {\fyT{Curator}\\[0.4mm]\strut skills from verified runs};
\draw[arr] (82,-40.3) -- (82,-47.5);
\draw[adv] (90,-47.5) -- (90,-40.3);
\node[model, minimum width=27.4mm, minimum height=9mm] (aud) at (112,-47.5)
  {\fyT{Auditor}\\[0.4mm]\strut report, \fyF{review.md}};
\draw[arr] (134,-40.3) -- (134,-47.5);
% ===================== legend =====================
\begin{scope}[shift={(0,-61)}]
  \filldraw[draw=fyBlue, fill=fyBlueFill, line width=0.6pt, rounded corners=0.8pt]
    (0,-1.1) rectangle ++(3.4,2.2);
  \node[lbl, anchor=west] (l1) at (4.3,0) {LLM agent};
  \filldraw[draw=fyGreen, fill=fyGreenFill, line width=0.8pt, rounded corners=0.8pt]
    ($(l1.east)+(2.0,-1.1)$) rectangle ++(3.4,2.2);
  \node[lbl, anchor=west] (l2) at ($(l1.east)+(6.3,0)$) {deterministic gate};
  \draw[sup] ($(l2.east)+(2.0,0)$) -- ++(5,0);
  \node[lbl, anchor=west] (l4) at ($(l2.east)+(7.9,0)$) {supervision};
  \draw[rej] ($(l4.east)+(2.0,0)$) -- ++(5,0);
  \node[lbl, anchor=west] (l5) at ($(l4.east)+(7.9,0)$) {repair};
  \draw[adv] ($(l5.east)+(2.0,0)$) -- ++(5,0);
  \node[lbl, anchor=west] (l6) at ($(l5.east)+(7.9,0)$) {advice / reuse};
\end{scope}
\end{tikzpicture}
\caption{\textbf{The \Fyan workflow for document-level formalization.}
\Specifier turns the source into theorem units and a document-level DAG; the
\Orchestrator{} schedules ready units based on theorem dependencies and regularly
checks ongoing proofs. Within each unit, models propose and deterministic gates decide: the
semantic audit (Figure~\ref{fig:semantic}) freezes the statement produced by
\Formalizer or returns it for repair, and the completion gate accepts the proof,
whose module is then reused downstream. Human guidance and \Curator skills are
advice only; \Auditor writes a report for human inspection.}
\label{fig:pipeline}
\end{figure}

\newtheorem{proposition}{Proposition}

\section{Method}
\label{sec:method}

\label{sec:decouple}
\label{sec:chain}

\Fyan\ maps PDF/TeX/Markdown sources into a dependency-aware Lean project with
provenance, semantic-review, proof, and completion records. It combines
\Specifier, \Reasoner, \Planner, \Formalizer, \Curator, and
\Auditor. Components communicate through persistent artifacts and explicit
control boundaries: models and experts may reason and guide, while validated
evidence determines stage advancement and acceptance. Details appear in
Appendix~\ref{sec:overview}.

\subsection{Workflow and Control Boundaries}
\label{sec:method:overview}
\label{sec:method:workflow}
\label{sec:method:artifacts}
\label{sec:chain:pipeline}

Figure~\ref{fig:pipeline} shows the end-to-end workflow. Intake preserves file
and page provenance, and \Specifier\ converts the reviewed source into
self-contained \code{problem.tex} units and a \emph{document-level dependency
DAG}. The \Orchestrator{} follows this DAG in topological order, starting a
theorem only after its dependencies are complete. Proved upstream units then
become importable Lean modules for downstream theorems, avoiding unnecessary
re-proving and enabling reuse across the project.

For each ready unit, \Reasoner, built on the open-source \QED{}
system~\citep{an2026qed}, develops a natural-language proof while respecting
the document-level dependency structure. The resulting \code{proof.md}
preserves source provenance and may reuse previously proved upstream units.
\Planner\ then translates this argument into a \emph{lemma-level DAG}, a
formalization plan, and a compiling \code{Statement.lean} scaffold. The
document-level DAG determines \emph{which theorem is formalized next}, whereas
the lemma-level DAG organizes \emph{how it is proved}. These artifacts
are kept read-only during proof search, preventing the prover from silently
altering the task it is meant to solve.

\Formalizer\ constructs Lean proofs through a plan--prove--review loop in
a Pi-based coding-agent runtime~\citep{pi-agent}. It implements the scaffold
in \code{Basic.lean}, may introduce helper lemmas and revise proof strategies, and can directly use
completed upstream modules. \Curator\ then mines verified runs for reusable
proof patterns, useful failures, and search strategies; these skills may guide
later runs but do not count as proof evidence. \Auditor\ combines deterministic
scans with a model-written \code{review.md} for human inspection.

Through the Dashboard, users can inspect progress, review blockers, and inject
mathematical or strategic guidance during a run; the \Orchestrator{} regularly
cross-checks ongoing proofs against the argument, specification, and plan, and
may activate critics or escalate questions. Such guidance is advice-only: it
cannot edit executable Lean, modify a protected statement, or force a stage
transition.

Across the workflow, persisted artifacts rather than transient agent
conversation carry state, and a gateway is the sole writer of stage state.
This separates \emph{capability} from \emph{acceptance}: reasoners, provers,
retrieved skills, \Orchestrator{} supervision, and human experts may all help find a proof,
but validated evidence alone determines what counts as a faithful and completed
formalization.
The next two subsections describe the semantic audit and the proof-transfer
and bridge obligations.

\subsection{Evidence-Grounded Semantic Auditing}
\label{sec:link1}
\begin{figure}[t]
\centering
\definecolor{fyBlue}{HTML}{2F6DB5}
\definecolor{fyBlueDark}{HTML}{1F4E8C}
\definecolor{fyBlueFill}{HTML}{EAF1FB}
\definecolor{fyBlueLine}{HTML}{9DBBE3}
\definecolor{fyGreen}{HTML}{2E8B57}
\definecolor{fyGreenDark}{HTML}{1E6B40}
\definecolor{fyGreenFill}{HTML}{E8F5EC}
\definecolor{fyGreenLine}{HTML}{9CCFB0}
\definecolor{fyRed}{HTML}{C0392B}
\definecolor{fyRedFill}{HTML}{FBEDEB}
\definecolor{fyInk}{HTML}{2B2B2B}
\definecolor{fyInkSoft}{HTML}{4B5563}
\definecolor{fyGray}{HTML}{9CA3AF}
\definecolor{fyGrayFill}{HTML}{F3F4F6}
\begin{tikzpicture}[x=1mm,y=1mm,
  font=\sffamily\scriptsize, text=fyInk,
  panel/.style={line width=0.6pt, rounded corners=1.2pt},
  bx/.style={line width=0.6pt, rounded corners=1.2pt},
  hd/.style={anchor=base west, inner sep=0pt, font=\sffamily\normalsize\bfseries},
  sub/.style={anchor=base west, inner sep=0pt, font=\sffamily\scriptsize\itshape},
  ttl/.style={anchor=base west, inner sep=0pt, font=\sffamily\footnotesize\bfseries\boldmath},
  lbl/.style={anchor=base west, inner sep=0pt},
  rlbl/.style={anchor=base east, inner sep=0pt},
  tag/.style={anchor=base east, inner sep=0pt, text=fyBlue, font=\sffamily\footnotesize},
  note/.style={anchor=base west, inner sep=0pt, font=\sffamily\scriptsize\itshape, text=fyInkSoft},
  qual/.style={anchor=base east, inner sep=0pt, font=\sffamily\scriptsize\itshape, text=fyInkSoft},
  chip/.style={anchor=base, rounded corners=0.8pt, minimum width=3mm, minimum height=2.7mm,
    inner sep=0pt, font=\sffamily\scriptsize\bfseries, text=white},
  ck/.style={draw=fyGreen, line width=0.7pt, line cap=round, line join=round},
  wire/.style={line width=0.7pt, draw=fyInkSoft},
  flow/.style={-{Latex[length=1.6mm]}, line width=0.7pt, draw=fyInkSoft},
  rep/.style={-{Latex[length=1.6mm]}, line width=0.7pt, draw=fyRed, rounded corners=1.2pt},
]
\newcommand{\fyto}{\,{\rightarrow}\,}
% ------------------------------------------------------------ layout (mm)
% columns: inputs 0-18.4 | judge panel 21.7-68.7 | validator panel 71.3-102.7 | bus 104.9 | outcome 106.9-139.6
% rows: header base 52.0, subtitle base 48.9, content 47.5 -> 7.7; main flow at y = 31.9
% ============================================================ panels
\draw[panel, draw=fyBlue,  fill=fyBlueFill]  (21.7,7.7) rectangle (68.7,55.8);
\draw[panel, draw=fyGreen, fill=fyGreenFill] (71.3,7.7) rectangle (102.7,55.8);
% ============================================================ column headers
\node[hd]                    at (0,52.0)     {Inputs};
\node[sub, text=fyInkSoft]   at (0,48.9)     {per theorem unit};
\node[hd, text=fyBlueDark]   at (22.9,52.0)  {{\rmfamily\scshape Fyan} Judge};
\node[sub, text=fyBlue]      at (22.9,48.9)  {proposes evidence};
\node[hd, text=fyGreenDark]  at (72.5,52.0)  {Validator};
\node[sub, text=fyGreen]     at (72.5,48.9)  {deterministic code decides};
\node[hd]                    at (106.9,52.0) {Outcome};
\node[sub, text=fyInkSoft]   at (106.9,48.9) {grade vs.\ threshold};
% ============================================================ inputs
\draw[bx, draw=fyGray, fill=fyGrayFill] (0,47.5) rectangle (18.4,39.9);
\node[ttl] at (1.1,44.4) {Source $N$};
\node[lbl] at (1.1,41.2) {\texttt{problem.tex}};
\draw[bx, draw=fyGray, fill=fyGrayFill] (0,37.2) rectangle (18.4,26.6);
\node[ttl] at (1.1,34.1) {Lean $S_1$};
\node[lbl] at (1.1,30.9) {candidate in};
\node[lbl] at (1.1,27.9) {library context};
% ============================================================ LLM Judge
% J1: source tree first (was: Lean-first inventory, S_1 ~> N)
\draw[bx, draw=fyBlueLine, fill=white] (22.9,47.5) rectangle (67.5,36.9);
\node[ttl] at (23.9,44.4) {Source tree};
\node[tag] at (66.5,44.4) {before Lean};
\node[lbl] at (23.9,41.2) {formula tree of $N$, literal quotes};
\node[lbl] at (23.9,38.2) {checked by an independent review};
% J2: slot pairing (was: reverse coverage, N ~> S_1)
\draw[bx, draw=fyBlueLine, fill=white] (22.9,35.7) rectangle (67.5,28.1);
\node[ttl] at (23.9,32.6) {Slot pairing};
\node[tag] at (66.5,32.6) {$N{\leftrightarrow}S_1$};
\node[lbl] at (23.9,29.4) {slots cut by code; unpaired = omission};
% J3: local relations
\draw[bx, draw=fyBlueLine, fill=white] (22.9,26.9) rectangle (67.5,9.2);
\node[ttl] at (23.9,23.8) {Local relation, grade $s_i$};
\foreach \k/\lev/\name/\gloss in {
    0/3/same/exact match,
    1/2/trivial\_equivalent/immediate,
    2/1/mutated\_*/justified change,
    3/0/unaligned/unresolved mismatch} {
  \node[chip, fill=fyBlue] at (25.4,20.5-3.1*\k) {\lev};
  \node[lbl] at (27.7,20.5-3.1*\k) {\texttt{\name}};
  \node[rlbl, text=fyInkSoft] at (66.5,20.5-3.1*\k) {\gloss};
}
% ============================================================ Validator
\draw[bx, draw=fyGreenLine, fill=white] (72.5,47.5) rectangle (101.5,25.5);
\foreach \k/\txt in {
    0/two-way coverage,
    1/direction vs.\ polarity,
    4/{scope, dependencies},
    5/evidence levels,
    6/required checks} {
  \node[lbl] at (76.1,44.7-3.0*\k) {\txt};
  \draw[ck] (73.6,44.7-3.0*\k+0.75) -- ++(0.5,-0.55) -- ++(0.95,1.1);
}
\node[lbl, text=fyInkSoft] at (78.6,44.7-3.0*2) {positive: $P_L\fyto P_S$};
\node[lbl, text=fyInkSoft] at (78.6,44.7-3.0*3) {negative: $P_S\fyto P_L$};
% grade box, centred on the gap between the two lower outcomes (y = 20.3)
\draw[bx, draw=fyGreen, fill=white] (72.5,24.3) rectangle (101.5,16.3);
\node[anchor=base, inner sep=0pt, font=\sffamily\footnotesize\bfseries] at (87.0,21.0)
  {grade $=\min_i s_i$};
\node[anchor=base, inner sep=0pt] at (87.0,17.8) {non-compensatory};
% form repair: the only way back to the Judge
\draw[flow, dashed, draw=fyGray] (73.2,11.9) -- (67.5,11.9);
\node[note] at (74.2,12.9) {malformed record:};
\node[note] at (74.2,10.0) {form repair only};
% ============================================================ Outcome
% O1 accept
\draw[bx, draw=fyGreen, fill=fyGreenFill] (106.9,47.5) rectangle (139.6,36.9);
\node[ttl] at (108.1,44.4) {Accept, freeze $S_1$};
\node[chip, fill=fyGreen] at (109.6,41.2) {3};
\node[lbl] at (111.8,41.2) {\texttt{aligned}};
\node[chip, fill=fyGreen] at (109.6,38.2) {2};
\node[lbl] at (111.8,38.2) {\texttt{aligned\_trivial}};
% default threshold
\draw[line width=0.6pt, draw=fyGray, dash pattern=on 1.6pt off 1.2pt] (106.9,34.9) -- (139.6,34.9);
\node[inner xsep=0.8mm, inner ysep=0.2mm, fill=white, font=\sffamily\scriptsize\itshape,
      text=fyInkSoft] at (123.25,34.9) {default threshold};
% O2 opt-in bridge
\draw[bx, draw=fyGreen, fill=white, dash pattern=on 2pt off 1.2pt] (106.9,32.9) rectangle (139.6,22.3);
\node[ttl] at (108.1,29.8) {Admit, freeze $S_1$};
\node[chip, fill=fyGreen] at (109.6,26.6) {1};
\node[lbl] at (111.8,26.6) {\texttt{modified}};
\node[qual] at (138.4,26.6) {opt-in};
\node[lbl] at (108.1,23.6) {bridge obligation $S_1\fyto S_0$};
% O3 reject
\draw[bx, draw=fyRed, fill=fyRedFill] (106.9,18.3) rectangle (139.6,7.7);
\node[ttl, text=fyRed] at (108.1,15.2) {Reject, repair};
\node[chip, fill=fyRed] at (109.6,12.0) {0};
\node[lbl] at (111.8,12.0) {\texttt{failed}};
\node[chip, fill=fyRed] at (109.6,9.0) {1};
\node[lbl] at (111.8,9.0) {\texttt{modified}};
\node[qual] at (138.4,9.0) {default};
% ============================================================ arrows
% inputs -> judge (merged) and judge -> validator: main flow at y = 31.9
\draw[wire] (18.4,43.7) -- (19.9,43.7) -- (19.9,31.9);
\draw[flow] (18.4,31.9) -- (21.7,31.9);
\fill[fyInkSoft] (19.9,31.9) circle (0.5);
\draw[flow] (68.7,31.9) -- (71.3,31.9);
% grade -> outcome bus
\draw[wire] (101.5,20.3) -- (104.9,20.3);
\draw[wire] (104.9,42.2) -- (104.9,13.0);
\fill[fyInkSoft] (104.9,20.3) circle (0.5);
\draw[flow] (104.9,42.2) -- (106.9,42.2);
\draw[flow] (104.9,27.6) -- (106.9,27.6);
\draw[flow] (104.9,13.0) -- (106.9,13.0);
% repair loop back to the autoformalizer
\draw[rep] (123.25,7.7) -- (123.25,5.0) -- (9.2,5.0) -- (9.2,26.6);
\node[fill=white, inner xsep=0.8mm, inner ysep=0.4mm, text=fyInkSoft] at (9.2,16.0) {Formalizer};
\node[fill=white, inner xsep=1mm, inner ysep=0.3mm, text=fyRed] at (66.6,5.0)
  {localized repair request: item at fault, direction of change, suggested fix};
% ============================================================ legend
\newcommand{\fySw}[2]{\tikz[baseline=-0.6ex]\draw[line width=0.6pt, rounded corners=1pt,
  draw=#1, fill=#2] (0,-1) rectangle (3.2,1);}
\node[anchor=base, inner sep=0pt] at (69.8,0.3) {%
  \fySw{fyBlue}{fyBlueFill}~LLM proposes\hspace{4mm}%
  \fySw{fyGreen}{fyGreenFill}~deterministic code decides\hspace{4mm}%
  \tikz[baseline=-0.6ex]\draw[flow] (0,0) -- (6,0);~data flow\hspace{4mm}%
  \tikz[baseline=-0.6ex]\draw[rep] (0,0) -- (6,0);~repair (below threshold)};
\end{tikzpicture}
\caption{\textbf{The \Fyan{} semantic audit.} The \Fyan{} Judge, an LLM, first extracts a
quoted, independently reviewed formula tree of $N$, then pairs the slots
(objects, hypotheses, claims) that code cuts from it and from $S_1$, with a
grade $s_i$ per paired group. Code derives polarity, checks the record, and routes $S_1$ by
the statement grade $\min_i s_i$: at or above the threshold (default 2) it is frozen for the
document-level project, else repaired; at threshold 1, grade 1 is admitted
with a Lean bridge obligation to a source-facing $S_0$.}
\label{fig:semantic}
\end{figure}
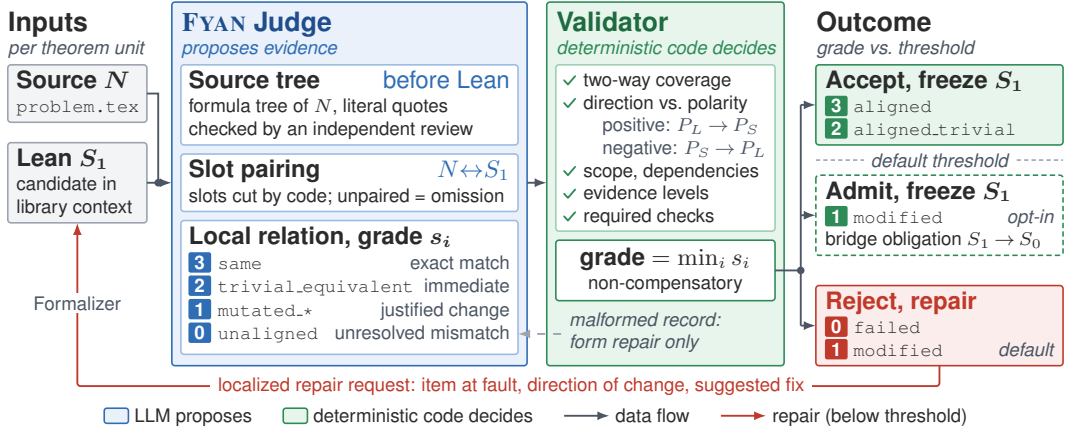

A formal statement may compile, and even be equivalent to its source as a
whole, while altering quantified objects, assumptions, or dependencies.
\Fyan{} therefore treats semantic faithfulness as an \emph{auditable relation}
between source and formal statement.
\iffalse
Two formal statements may be mathematically equivalent, and both may even be
provable, while only one faithfully represents the original problem. An
equivalent reformulation can change the quantified objects, alter assumptions
or witness dependencies, or encode a different mathematical construction.
Fyan therefore treats semantic faithfulness as an \emph{auditable relation}
between the source and the formal statement, rather than as successful
compilation or whole-statement equivalence. The audit combines local structural
correspondence with logical evidence for any non-identical change.
\fi

\paragraph{Structured semantic record.}
The \Fyan{} Judge compares the source $N$ with the candidate Lean statement $S_1$
in its actual library environment, starting from the source. Before any Lean
is shown, one call extracts a formula tree of $N$ whose definitions and atoms
are anchored by literal source quotations, and an independent call reviews it.
The host reads the Lean tree from the elaborated declaration, with the
definitions it reaches, and cuts both trees into the same \emph{slots}: binder
domains (\emph{objects}), hypothesis conjuncts (\emph{hypotheses}), goal
conjuncts (\emph{claims}), and used definitions, each with host-computed scope
and polarity. The \Fyan{} Judge fills one row per group of paired slots with its
semantic relation and comparison evidence. Restrictions carried by types,
subtypes, instances, or implicit parameters are treated as mathematical
content.

Because both sides are cut into slots, source content that has disappeared
entirely remains visible: it leaves a source slot that no row pairs, which the
host scores as an omission. For example, if the source assumes both $x>0$ and
$x<1$, while the Lean statement contains only $x>0$, the remaining condition is
paired and the unpaired slot for $x<1$ is recorded as an omitted hypothesis.

\paragraph{Local semantic relations.}
Once the slots have been paired, \Fyan{} classifies each
local correspondence into one of the four relation families of
Table~\ref{tab:semantic-relations}.

\begin{table}[t]
\centering
\normalsize
\caption{Local semantic relations. The \texttt{mutated\_*} family includes
equivalent, stronger, and weaker changes; one-way changes are admissible only
in the appropriate logical polarity.}
\label{tab:semantic-relations}
\begin{tabular}{@{}p{0.28\linewidth}p{0.55\linewidth}c@{}}
\toprule
Relation & Interpretation & Grade \\
\midrule
\texttt{same}
& Same mathematical content and structure
& 3 \\

\texttt{trivial\_equivalent}
& Immediate equivalence after notation or checked definitions
& 2 \\

\texttt{mutated\_*}
& Substantive change supported by a bounded equivalence or directional proof
& 1 \\

\texttt{unaligned}
& Unsupported, mismatched, unresolved, or requiring a complex comparison
& 0 \\
\bottomrule
\end{tabular}
\end{table}
\iffalse
The direction of a substantive change depends on where it occurs in the
complete source formula. Let $P_S$ and $P_L$ be corresponding source and Lean
predicates. Fyan propagates polarity from the statement root: conjunction,
disjunction, and fixed-domain quantifiers preserve polarity, whereas negation
and implication antecedents reverse it. The admissible one-way directions are
\[
\text{positive: } P_L \rightarrow P_S,
\qquad
\text{negative: } P_S \rightarrow P_L.
\]
Thus a candidate conclusion may be stronger, whereas a candidate assumption
may be weaker. For instance,
\[
(A\land H)\rightarrow G
\quad\longrightarrow\quad
A\rightarrow G
\]
removes a conjunct from a negatively occurring assumption and still permits
recovery of the source theorem. In contrast,
\[
A\rightarrow(G\land H)
\quad\longrightarrow\quad
A\rightarrow G
\]
drops part of a positive conclusion and is not recoverable in general.
Omissions are handled by this same polarity principle rather than by a
separate heuristic for ``missing assumptions'' or ``missing goals''.
\fi
The admissibility of a substantive change depends on its polarity in the
complete source formula. For corresponding source and Lean predicates $P_S$
and $P_L$, conjunction, disjunction, and fixed-domain quantifiers preserve
polarity, while negation and implication antecedents reverse it. Hence the
admissible one-way directions are
\[
\text{positive: } P_L \rightarrow P_S,
\qquad
\text{negative: } P_S \rightarrow P_L.
\]
Thus candidate conclusions may be stronger, while candidate assumptions may
be weaker. For example, replacing $(A\land H)\rightarrow G$ by
$A\rightarrow G$ is recoverable, whereas replacing
$A\rightarrow(G\land H)$ by $A\rightarrow G$ is not in general.
Omissions are handled by the same polarity principle.

\iffalse
\paragraph{Deterministic judgment.}
The language model produces the structured evidence record but does not choose
the final judgment. A deterministic validator checks source coverage,
relation admissibility, scope and dependency consistency, evidence thresholds,
and the required semantic checks. If the valid record contributes
$s_i\in\{0,1,2,3\}$, Fyan computes
\[
\operatorname{Score}(R)
=
\min\bigl(\{3\}\cup\{s_i:i\in I(R)\}\bigr),
\]
yielding
\[
3:\texttt{aligned},\qquad
2:\texttt{aligned\_trivial},\qquad
1:\texttt{modified},\qquad
0:\texttt{failed}.
\]
The minimum is deliberately non-compensatory: one unresolved semantic defect
cannot be hidden by many correctly aligned items, and unknown or overly
complex comparisons fail rather than being replaced by a model confidence
score. A \texttt{modified} judgment alone is not sufficient for acceptance;
Next we show how admissible local changes induce explicit
proof-transfer obligations and, when necessary, a Lean-checkable bridge back
to the source specification.
\fi
\paragraph{Deterministic judgment.}
The language model fills the structured evidence record but does not choose
the verdict. A deterministic validator checks coverage, relation admissibility,
scope and dependencies, evidence thresholds, and required semantic checks.
Each valid row receives a grade $s_i\in\{0,1,2,3\}$, and \Fyan{} computes
the statement grade
\[
\operatorname{grade}(R)
=
\min\bigl(\{3\}\cup\{s_i:i\in I(R)\}\bigr),
\]
corresponding to $3$ (\texttt{aligned}), $2$ (\texttt{aligned\_trivial}),
$1$ (\texttt{modified}), and $0$ (\texttt{failed}). The minimum is
non-compensatory: a single unresolved, unknown, or overly complex comparison
cannot be offset by well-aligned items. A \texttt{modified} verdict records a substantive but locally admissible
change; Section~\ref{sec:transfer} shows how \Fyan{} verifies its whole-statement
recoverability through explicit proof-transfer obligations and, when needed,
a Lean-checkable bridge.

\subsection{Proof Transfer and Bridge Obligations}
\label{sec:transfer}
\label{sec:method:audit-policy}
\label{sec:method:recovery}
\label{sec:method:acceptance}
\label{sec:link2}

A directionally admissible change is not thereby faithful. \Fyan{} instead
asks whether a proof of the candidate statement can be transferred back to a
Lean rendering of the source specification.

\iffalse
\paragraph{Global proof transfer.}
Let
\[
S_0=\forall x:X,\; A_0(x)\rightarrow G_0(x)
\]
be a Lean rendering of the source specification, and let
\[
S_1=\forall y:Y,\; A_1(y)\rightarrow G_1(y)
\]
be the audited candidate statement. A recoverable modification must supply
three pieces of evidence:
\[
e:X\rightarrow Y,
\]
which represents each admissible source object in the candidate domain,
together with
\[
a:\forall x:X,\; A_0(x)\rightarrow A_1(e(x))
\]
and
\[
g:\forall x:X,\;
A_0(x)\rightarrow G_1(e(x))\rightarrow G_0(x).
\]
The map $a$ shows that the source assumptions provide everything needed to
apply the candidate theorem, while $g$ shows that the candidate conclusion is
sufficient to recover the source conclusion. Hence every proof $p:S_1$
induces
\begin{equation}
\label{eq:transfer-main}
\kappa(p)(x)(h)
=
g\,x\,h\bigl(p\,(e(x))\,(a\,x\,h)\bigr),
\end{equation}
and therefore a canonical proof transformer
$\kappa:S_1\rightarrow S_0$.
These three requirements correspond to the
\texttt{object\_coverage}, \texttt{hypothesis\_transfer}, and
\texttt{claim\_transfer} obligations recorded by the audit.
\fi
\paragraph{Global proof transfer.}
Let
\[
S_0:=\forall x:X,\;A_0(x)\rightarrow G_0(x),
\qquad
S_1:=\forall y:Y,\;A_1(y)\rightarrow G_1(y)
\]
denote the source-facing Lean statement and the audited candidate,
respectively. A recoverable modification supplies
\[
e:X\rightarrow Y,\qquad
a:\forall x:X,\;A_0(x)\rightarrow A_1(e(x)),\qquad
g:\forall x:X,\;A_0(x)\rightarrow G_1(e(x))\rightarrow G_0(x).
\]
Here $e$ maps each admissible source object into the candidate domain,
$a$ provides the assumptions needed to apply the candidate theorem, and
$g$ recovers the source conclusion from the candidate conclusion. Hence every
proof $p:S_1$ induces the canonical transfer
\begin{equation}
\label{eq:transfer-main}
\kappa:S_1\rightarrow S_0,\qquad
\kappa(p)(x)(h)
=
g\,x\,h\bigl(p\,(e(x))\,(a\,x\,h)\bigr).
\end{equation}
These three requirements correspond to the
\texttt{object\_coverage}, \texttt{hypothesis\_transfer}, and
\texttt{claim\_transfer} obligations recorded by the audit.

\paragraph{From local evidence to global transfer.}
The condition above is stated at the level of the complete theorem, whereas the
\Fyan{} Judge in Section~\ref{sec:link1} records local correspondences. For the logical fragment
covered by our audit, these two levels are compatible: scope-aware local
certificates can be composed along the logical structure of the statement
according to their polarity. In particular,
certificates for aligned objects and admissible local changes under conjunction,
disjunction, implication, and quantifiers can be combined to construct the
global maps $e$, $a$, and $g$, provided that coverage, scope, and witness
dependencies are preserved. Thus the directional rules used by the \Fyan{}
Judge are not merely local heuristics; they form a sufficient certificate
discipline for whole-statement proof transfer
(Theorem~\ref{thm:compositional-transfer}, Appendix~\ref{app:transfer:global}).
\iffalse
\paragraph{Bridge obligations.}
A recoverable change is still a semantic change. Therefore, Fyan does not
upgrade a \texttt{modified} statement to \texttt{aligned} simply because a
transfer exists. Instead, once a non-identical candidate $S_1$ is accepted,
Fyan freezes it and creates an explicit bridge obligation to the
source-facing statement $S_0$. The prover must then construct a Lean theorem
showing that every proof of $S_1$ can be converted into a proof of $S_0$ using
the required object, hypothesis, and claim transfers. The bridge is checked by
Lean and stored as part of the project, so the accepted semantic difference is
resolved by a formal proof rather than by the Judge's natural-language
explanation alone. The source-facing bridge statement is also re-audited
before the obligation is considered discharged.
\fi
\paragraph{Bridge obligations.}
A recoverable change is still a semantic change. \Fyan{} does not
reclassify a \texttt{modified} statement as \texttt{aligned} merely because a
transfer exists. By default, grades~2 and~3 are accepted directly, whereas a
grade~1 statement is returned for repair. If the acceptance threshold is lowered
to grade~1, the pipeline instead freezes $S_1$ and records a bridge obligation to
a source-facing statement $S_0$. The prover must discharge this obligation by
establishing $S_1 \rightarrow S_0$ in Lean, and the bridge statement is itself
re-audited before being frozen. Thus any accepted semantic deviation is backed
by a formal proof rather than by the \Fyan{} Judge's natural-language evidence
alone.\Fyan{} therefore keeps semantic classification separate from proof-theoretic recoverability

The bridge certifies only the formal implication $S_1 \rightarrow S_0$; whether
$S_0$ faithfully represents the source remains the responsibility of the
semantic audit in Section~\ref{sec:link1}. Appendix~\ref{app:transfer} gives the
full compositional transfer theorem and its proof.
\providecommand{\odeReplay}{17}

\section{Evaluation}
\label{sec:experiments}
\label{sec:exp}
\label{sec:results}

We ask whether the harness improves verified proof construction with the model
and verifier held fixed, what the semantic audit reveals beyond direct judgment
and reference equivalence, and whether \Fyan{} sustains a dependency-linked
formalization project.

\subsection{Experimental Setup}
\label{sec:setup}
\label{sec:exp:setup}
\label{sec:experimental-setup}

\paragraph{Stage-wise evaluation.}
FormalTCS~\citep{wang2026formaltcs} contains \Ntcs{} theorems and separately evaluates statement
translation (NT2FT), natural-language proof generation (T2NP), and formal proof
construction (FT2FP), preventing errors from propagating across stages. We
compare \Fyan{} with a general-purpose agent harness using the same
DeepSeek-V4.1-Flash model, Lean environment, and released graders, with one
output per item. NT2FT is evaluated by BEq$^+$, T2NP by the benchmark rubric
judge, and FT2FP by strict Lean verification, including compilation, signature
preservation, axiom checking, and fresh kernel replay. Full configurations are
given in Appendix~\ref{app:evaluation}. A stand-alone run of the \Formalizer{} proof engine on ProofNet\# is
reported in Appendix~\ref{app:proofnet}.

\paragraph{Semantic-audit evaluation.}
We evaluate the relation-based audit on \ccTN{} ConsistencyCheck statements,
balanced between expert-labelled consistent and inconsistent pairs. The baseline
is a direct judge using the same model and per-call reasoning setting. The audit
instead performs source extraction, local alignment, and deterministic
validation, so the comparison isolates the auditing procedure. Details appear
in Appendix~\ref{app:evaluation}.

\paragraph{Document-level evaluation.}
We use ODENumLib, a dependency-linked Lean development built with \Fyan{}, to
evaluate the workflow beyond isolated theorem proving. The project contains
multiple theorem units that reuse proved upstream results, allowing us to examine
project-scale dependency management and theorem reuse. We report its number of
units, Lean artifact size, completion and replay status under the pinned
environment.
\subsection{Verified Proof Production}
\label{sec:main}
\label{sec:exp:axes}

\begin{table}[!ht]
\centering
\normalsize
\caption{\textbf{Controlled comparison on FormalTCS}
(\Ntcs{} items; one output per item).}
\label{tab:fts-main}
\setlength{\tabcolsep}{5pt}
\begin{tabular}{@{}llccc@{}}
\toprule
Task & Metric & Control & \Fyan{} & Difference \\
\midrule
FT2FP & strict Lean pass
  & \ftfpCtl/\Ntcs{} (\ftfpCtlPct\%)
  & \textbf{\ftfpFyan/\Ntcs{} (\ftfpFyanPct\%)} & $+17$ \\
T2NP & rubric mean
  & \tnpCtl & \textbf{\tnpFyan} & $+0.350$ \\
NT2FT & BEq$^+$ equivalent
  & \ntftCtl/\Ntcs{} (\ntftCtlPct\%)
  & \textbf{\ntftFyan/\Ntcs{} (\ntftFyanPct\%)} & $+3$ \\
\bottomrule
\end{tabular}
\end{table}

\iffalse
\Fyan{} improves both formal and informal proving (Table~\ref{tab:fts-main}).
The NT2FT counts are low for both arms because reference matching, rather than the
statements, sets the ceiling: BEq$^+$ does not identify independently declared
types, which most FormalTCS references introduce (Section~\ref{sec:analysis}).
This ceiling binds every system: the strongest result reported by the
benchmark authors, from Claude Opus~5, is \pubBest\% with eight samples per
item~\citep{wang2026formaltcs}, whereas both arms here submit a single output
from a smaller model. We therefore assess statement faithfulness directly with
the semantic audit.
\fi
\Fyan{} improves both formal proof construction and natural-language reasoning
(Table~\ref{tab:fts-main}). The NT2FT scores remain low for both systems, largely
because BEq$^+$ requires the generated statement to match the reference
formalization closely enough for equivalence to be established. In particular,
it does not identify independently declared types, which appear in many
FormalTCS references (Section~\ref{sec:analysis}). This limitation affects all
systems evaluated with the same metric: the strongest result reported by the
benchmark authors, obtained with Claude Opus~5 and eight samples per item, is
\pubBest\%~\citep{wang2026formaltcs}, whereas both systems here produce a single
output using a smaller model. We therefore evaluate statement faithfulness
separately using the semantic audit.

\subsection{Semantic Auditing for Faithfulness}
\label{sec:analysis}
\label{sec:grade}
\label{sec:beq}
\label{sec:exp:beq}
\label{sec:reliability}
\label{sec:exp:judge}
\label{sec:analysis:judge}
\label{sec:analysis:slots}

We next ask whether structured semantic auditing detects errors that are missed
by a direct, one-shot judgment. What a gate in front of a reusable library must
deliver is reliability: an unfaithful formalization has to be caught and sent
back to be rebuilt. The two kinds of error are therefore far from symmetric. A
false alarm, even on a statement that is in fact equivalent to its source,
costs one more autoformalization round: the statement returns with a localized
repair request and the autoformalizer revises it. A missed inconsistency is
silent: the altered statement is frozen, proved, and imported by downstream
units, no later check revisits it, and every proof that builds on it certifies
a theorem other than the one the document states. In a large project the first
cost is routine and the second is unacceptable. We therefore treat recall on
inconsistent statements as the primary metric, well ahead of precision.

\begin{table}[!ht]
\centering
\small
\caption{\textbf{Detecting inconsistent statements on ConsistencyCheck}
(\ccTN{} items; inconsistent is the positive class). Missed errors are
inconsistent statements judged consistent. The verified labels come from
re-verifying all \ccAdjN{} statements against a single criterion
(Section~\ref{sec:analysis}): \ccAdjRelabel{} statements that the original
labels call consistent become inconsistent, and none moves the other way. Bold marks the better method in each
column; shaded rows are the \Fyan{} judge. Confusion counts, F1, and bootstrap
intervals are in Appendix~\ref{app:cc}.}
\label{tab:cc}
\setlength{\tabcolsep}{6pt}
\renewcommand{\arraystretch}{1.18}
\begin{tabular}{@{}lccccc@{}}
\toprule
 & Recall $\uparrow$ & Missed errors $\downarrow$ & Precision $\uparrow$ & False alarms $\downarrow$ & Accuracy $\uparrow$ \\
\midrule
\multicolumn{6}{@{}l}{\emph{Original labels}} \\
\quad Direct judge
  & \ccTDirR
  & \ccTDirFN
  & \textbf{\ccTDirP}
  & \textbf{\ccTDirFP}
  & \textbf{\ccTDirAcc} \\
\rowcolor{ccAuditRow}
\quad \Fyan{} judge
  & \textbf{\ccTAuditR}
  & \textbf{\ccTAuditFN}
  & \ccTAuditP
  & \ccTAuditFP
  & \ccTAuditAcc \\
\midrule
\multicolumn{6}{@{}l}{\emph{Verified labels}} \\
\quad Direct judge
  & \ccAdjDirR
  & \ccAdjDirFN
  & \textbf{\ccAdjDirP}
  & \textbf{\ccAdjDirFP}
  & \ccAdjDirAcc \\
\rowcolor{ccAuditRow}
\quad \Fyan{} judge
  & \textbf{\ccAdjAuditR}
  & \textbf{\ccAdjAuditFN}
  & \ccAdjAuditP
  & \ccAdjAuditFP
  & \textbf{\ccAdjAuditAcc} \\
\bottomrule
\end{tabular}
\end{table}

\begin{figure}[!ht]
\centering
\begin{tikzpicture}[x=1mm,y=1mm,font=\sffamily\scriptsize]
% one count = 0.9 mm; panel origins
\def\u{0.56}
\def\xa{29}
\def\xb{84}
\def\h{1.7}
% gridlines and ticks
\foreach \x in {\xa,\xb} {
  \foreach \t in {0,20,40,60,80} {
    \draw[black!12, line width=0.4pt] ({\x+\t*\u},4.6) -- ({\x+\t*\u},33.2);
    \node[text=black!55] at ({\x+\t*\u},2.9) {\t};
  }
  \draw[black!45, line width=0.5pt] (\x,4.6) -- (\x,33.2);
}
% panel titles
\node[anchor=base west, font=\sffamily\footnotesize] at (\xa,36.2) {\textbf{Missed inconsistencies}\enspace{\scriptsize\color{black!60}fewer is better}};
\node[anchor=base west, font=\sffamily\footnotesize] at (\xb,36.2) {\textbf{False alarms}\enspace{\scriptsize\color{black!60}fewer is better}};
% row labels
\node[anchor=base west, font=\sffamily\scriptsize\bfseries] at (0,30.4) {Original labels};
\node[anchor=west] at (1.5,27) {Direct judge};
\node[anchor=west] at (1.5,22.6) {\Fyan{} judge};
\node[anchor=base west, font=\sffamily\scriptsize\bfseries] at (0,15.9) {Verified labels};
\node[anchor=west] at (1.5,12.5) {Direct judge};
\node[anchor=west] at (1.5,8.1) {\Fyan{} judge};
% missed inconsistencies
\fill[ccOrange] (\xa,{27-\h}) rectangle ({\xa+27*\u},{27+\h});
\node[anchor=west] at ({\xa+27*\u+0.6},27) {27};
\fill[ccBlue] (\xa,{22.6-\h}) rectangle ({\xa+19*\u},{22.6+\h});
\node[anchor=west] at ({\xa+19*\u+0.6},22.6) {19};
\fill[ccOrange] (\xa,{12.5-\h}) rectangle ({\xa+75*\u},{12.5+\h});
\node[anchor=west] at ({\xa+75*\u+0.6},12.5) {75};
\fill[ccBlue] (\xa,{8.1-\h}) rectangle ({\xa+46*\u},{8.1+\h});
\node[anchor=west] at ({\xa+46*\u+0.6},8.1) {46};
% false alarms (audit split by grade: 1 / 0 / no verdict; 0.35 mm surface gaps)
\fill[ccOrange] (\xb,{27-\h}) rectangle ({\xb+8*\u},{27+\h});
\node[anchor=west] at ({\xb+8*\u+0.6},27) {8};
\fill[ccGradeOne] (\xb,{22.6-\h}) rectangle ({\xb+18*\u-0.35},{22.6+\h});
\fill[ccGradeZero] ({\xb+18*\u},{22.6-\h}) rectangle ({\xb+36*\u-0.35},{22.6+\h});
\fill[ccNoVerdict] ({\xb+36*\u},{22.6-\h}) rectangle ({\xb+39*\u},{22.6+\h});
\node[anchor=west] at ({\xb+39*\u+0.6},22.6) {39};
\node[anchor=west] at ({\xb+0.6},12.5) {0};
\fill[ccGradeOne] (\xb,{8.1-\h}) rectangle ({\xb+5*\u-0.35},{8.1+\h});
\fill[ccGradeZero] ({\xb+5*\u},{8.1-\h}) rectangle ({\xb+9*\u-0.35},{8.1+\h});
\fill[ccNoVerdict] ({\xb+9*\u},{8.1-\h}) rectangle ({\xb+10*\u},{8.1+\h});
\node[anchor=west] at ({\xb+10*\u+0.6},8.1) {10};
% legend
\newcommand{\ccSw}[1]{\tikz[baseline=-0.6ex]\fill[#1] (0,-0.9mm) rectangle (2.8mm,0.9mm);}
\node[anchor=base west] at (0,-1.6) {%
  \ccSw{ccOrange}~Direct judge\hspace{4mm}%
  \ccSw{ccBlue}~\Fyan{} judge\hspace{6mm}%
  \Fyan{} judge false alarms:\hspace{1.5mm}%
  \ccSw{ccGradeOne}~grade~1\hspace{3mm}%
  \ccSw{ccGradeZero}~grade~0\hspace{3mm}%
  \ccSw{ccNoVerdict}~unknown};
\end{tikzpicture}
\caption{\textbf{Errors on ConsistencyCheck under the original and the verified
labels} (Table~\ref{tab:cc}). The \Fyan{} judge's false alarms are split by grade;
\emph{unknown} marks runs without a valid record.}
\label{fig:cc-errors}
\end{figure}
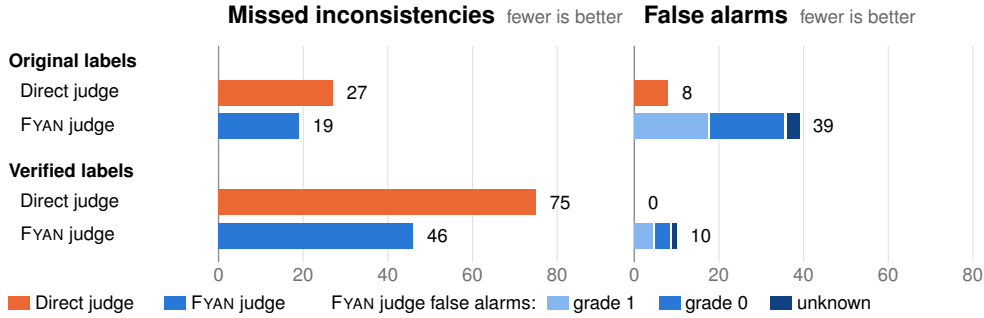

\paragraph{Fewer missed inconsistencies.}
Under the original labels the audit already leads on recall
(Table~\ref{tab:cc}); with its runs without a valid record counted as
acceptances, recall is \ccTOpenR{} against \ccTDirR{}
(Appendix~\ref{app:audit-errors}). Each
rejection also names the source object, hypothesis, claim, or scope at fault
rather than returning a single bit. Of the \ccTOnlyAudit{} items that
only the audit classifies correctly under these labels, \ccTStruct{} are structural errors in
scope, quantification, or intervals that a holistic reading passes over, of the
kind shown by the development item in case~(a) of Table~\ref{tab:audit-cases}.

\paragraph{Grade-1 false alarms.}
Of the \ccTAuditFP{} statements that the original labels call consistent but
the audit rejects, \ccTPolicy{} receive grade~1. Re-verification (below) finds
that \ccVerGOneNon{} of these are not the text's statement: \ccVerGOneArg{}
are equivalent to it only by an argument---an enlarged domain that the other
hypotheses restore, a representation that needs a translation, a codomain
widened from a subspace to the whole space---and \ccVerGOneOther{} rest on a
reading that the text does not state or differ from it outright; in
\code{exercise\_1\_13a} the Lean statement omits the connectedness of the
domain and is false. The original labels accept these statements; the audit,
which certifies alignment, grades them as \code{modified}. The other
\ccVerGOneSame{} grade-1 reports and \ccVerRestSame{} further rejections are
false alarms of the audit itself (Appendix~\ref{app:audit-errors}).

\paragraph{Verified labels.}
We re-verified all \ccAdjN{} statements against one criterion: a formal
statement is consistent only if it restates the text, differing at most by
conversions a reader sees at a glance. Of the 150 statements that the original
labels call consistent, \ccAdjSame{} meet it; \ccAdjEqNotSame{} are equivalent
to the text only by an argument, \ccAdjUnstated{} rely on a reading or an
assumption that the text does not state, and \ccAdjNotEq{} are not equivalent
to it. In \code{exercise\_3\_2\_21}, for instance, the hypothesis that $\sigma$
and $\tau$ disturb no common element becomes the requirement that every element
is moved by exactly one of them, which together with $\sigma\tau=e$ is
unsatisfiable, so the Lean theorem is vacuous. None of the 150 statements
labelled inconsistent meets the criterion, as expected: the criterion is
stricter than the benchmark's notion of consistency, so an inconsistent label
could move only if it were simply wrong. Relabelling the \ccAdjRelabel{}
statements as inconsistent gives the verified labels (Table~\ref{tab:cc} and
Figure~\ref{fig:cc-errors}). The criterion is the one the direct judge is
instructed to apply: the same objects and domains, hypotheses, and conclusion,
with nothing missing, added, weakened, or strengthened
(Appendix~\ref{app:cc}). It nevertheless accepts \ccAdjDirAcceptsRelabel{} of
the \ccAdjRelabel{} relabelled statements, whereas the audit, which encodes the
criterion in its closed list of conversions and its grade threshold, rejects
\ccAdjAuditRejectsRelabel{}. Under the verified labels the audit leads on recall and accuracy, and its
false alarms fall from \ccTAuditFP{} to \ccAdjAuditFP{} (Table~\ref{tab:cc},
Figure~\ref{fig:cc-errors}). It is correct alone on \ccAdjMcAudit{} statements and the direct judge on
\ccAdjMcDir{} (exact McNemar $p=\ccAdjMcP$; Appendix~\ref{app:audit-errors}).
Every one of the \ccAdjAuditFN{} inconsistencies that the audit misses is a
grade-2 acceptance, in which the Judge recorded a difference as one of the
listed conversions, and all \ccAdjGradeThree{} statements graded~3 are
consistent. Grade~3 is rare, so the audit's remaining errors sit at the grade-2
boundary, in the list of admissible conversions.

The criterion decides cases in both directions. In
\code{exercise\_3\_1\_22a} the text asks to show that the intersection of two
normal subgroups is normal and the Lean states
\code{Subgroup.Normal (H}~$\sqcap$~\code{K)}; the experts and the direct judge
accept it, while the audit grades it~1 because its Judge records Mathlib's
infimum of subgroups as an argued equivalence to intersection rather than as
its standard name, and the statement remains a false alarm. In
\code{exercise\_7\_11} the text asks to prove that every normal operator on a
complex inner-product space has a square root, and the Lean adds
\code{[FiniteDimensional}~$\mathbb{C}$~\code{V]}. The experts and the direct
judge accept it, reading in the finite dimension that the source textbook
assumes throughout; the audit rejects it because that hypothesis is not in the
statement. By the criterion the audit is right: a condition that is not in the
text must not be assumed, however standard it is in the text's source.
Statements that are stronger than the text, or equivalent to it only by a
non-trivial argument, are likewise not the same statement, even when a proof of
one yields the other. Were they counted as the same, one would have to decide
how long a bridge proof may be and still witness equivalence, and no such line
can be drawn without ambiguity. The audit therefore accepts only agreement that
is evident, grades argued equivalences and one-way changes~1, and leaves them
to a Lean-checked bridge rather than to the Judge's word; its aim is to make
every acceptance as certain as possible, with one standard fixed before any
statement is judged and applied uniformly.

% One real audit record per yardstick (direct LLM judge, expert labels, BEq+).
\begin{table}[t]
\centering
\small
\caption{\textbf{One audit, three yardsticks} (real records, abridged). Each
yardstick returns one bit; the audit localizes the difference, grades it, and
routes the statement. Relations and quoted arguments come from the \Fyan{}
Judge; polarity and grades are computed by code, and ($-$) marks negative
polarity. Case~(a) is from the diagnostic pool, not the evaluation sample.}
\label{tab:audit-cases}
\setlength{\tabcolsep}{4pt}
\renewcommand{\arraystretch}{1.1}
\begin{tabular}{@{}
  >{\raggedright\arraybackslash}p{2.3cm}
  >{\raggedright\arraybackslash}p{2.35cm}
  >{\raggedright\arraybackslash}p{6.85cm}
  >{\raggedright\arraybackslash}p{1.55cm}@{}}
\toprule
Case & Yardstick & Audit record (abridged) & Grade \\
\midrule
(a) Arithmetic series (miniF2F)
& Direct judge: \emph{consistent}\newline
  Experts: \emph{inconsistent}
& \code{$\sum$ k in range 5, a + k * d = 70} parses as $(\sum_k a)+kd$ with a
  free \code{k}. Hypothesis slot ($-$), \code{unaligned}: ``$5a+kd=70$ constrains a
  different object; $a=28$, $kd=-70$ satisfies the candidate's hypotheses but
  not the source.''
& \textbf{0}~$\rightarrow$ repair \\
\addlinespace
(b) Floor equation (miniF2F)
& Experts: \emph{consistent}\newline
  Direct judge: \emph{consistent}
& Source: $a=p/q$ with $p,q$ relatively prime positive integers. Binder slot
  \code{a :}~$\mathbb{Q}$ ($-$) also admits $a\le 0$: \code{mutated\_weaker},
  so at negative polarity ``the candidate statement implies the source
  statement (a proof transfers), but it says something else''; repair:
  restrict \code{a} to the positive rationals.
& \textbf{1}~$\rightarrow$ repair or bridge \\
\addlinespace
(c) Product measure bounds (FormalTCS)
& BEq$^+$: \code{not\_proven}
& Every row names its conversion: $n\ge 2$ as \code{n :}~$\mathbb{N}$,
  $n\ge2$ (condition moved into a type); $D_1,\dots,D_n$ as
  \code{Fin n}~$\to$~\code{ProbabilityMeasure}~$\mathbb{R}$ (reindexing);
  the product law as \code{productLaw} (Mathlib vs.\ local definition); both
  bounds \code{same} or reindexed. Premise: a division relies on Lean's total
  division, as the text never excludes zero.
& \textbf{2}~$\rightarrow$ accept \\
\bottomrule
\end{tabular}
\end{table}

\paragraph{One audit, three yardsticks.}
Table~\ref{tab:audit-cases} sets the audit against the three yardsticks in
current use, on one real record each. The direct judge passes~(a) as a whole, although a
misplaced summation delimiter leaves a free variable and makes the candidate
false; the audit exhibits an instance on which a hypothesis slot and its source
sentence disagree. The experts pass~(b), which the audit grades~1 with a
proposed repair. Such cases are where
a binary label must make a policy choice: of the \ccTGradeOne{} statements
graded~1, the experts call \ccTPolicy{} consistent and \ccTGradeOneInc{}
inconsistent, whereas the grade names the modification and leaves the decision
to an explicit threshold. BEq$^+$ does not certify~(c), although each
conversion is routine and named in the record, because it does not identify
independently declared types. Of the \Ntcs{}
FormalTCS references, \genRefs{} declare a \code{structure}, \code{inductive},
or \code{class}, and all \ntftCtl{} NT2FT outputs of the control harness that
BEq$^+$ certifies lie among the other \defRefs{} (one-sided Fisher
$p=\genFisher$; Appendix~\ref{app:beq}). A score against any of these
yardsticks therefore cannot by itself serve as the acceptance criterion of a
library; the graded audit is designed to be one.

\subsection{Document-Level Formalization}
\label{sec:document}
\label{sec:results:scale}

ODENumLib, built with \Fyan{}, formalizes the numerical analysis of ordinary
differential equations in \odeDone{} dependency-linked theorem units
(Table~\ref{tab:ode-main}): discrete Gr\"onwall inequalities, Taylor and
interpolation remainders, local existence and uniqueness, and the error
analysis of standard solvers.

\begin{table}[!ht]
\centering
\normalsize
\caption{\textbf{ODENumLib at the document level.}
A unit is formalized when it contains no \code{sorry} and introduces no new
axioms in its original environment. Replay is measured under the currently
pinned Lean and Mathlib revisions.}
\label{tab:ode-main}
\setlength{\tabcolsep}{6pt}
\begin{tabular}{@{}lr@{}}
\toprule
Measure & Result \\
\midrule
Formalized theorem units & \odeDone{} \\
Replaying under the current environment & \odeReplay{} \\
Requiring library porting & 3 \\
Lean artifact & \odeLines{} lines; \odeThm{} theorems; \odeLem{} lemmas \\
\bottomrule
\end{tabular}
\end{table}

Several units are substantial formalizations in their own right. They include
error bounds for the whole family of $\theta$-methods (1,551 lines of Lean),
the moment conditions that characterize consistency of linear multistep
methods (1,439 lines), stability regions and L-stability (713 lines), a
multistep Gr\"onwall inequality, and the well-posedness and convergence of the
implicit backward-Euler and trapezoidal rules. Units are formalized in
dependency order and import the modules they depend on rather than re-proving
them; the complete inventory is in Appendix~\ref{app:ode}.

\section{Conclusion and Limitations}
\label{sec:limits}
\label{sec:discussion}
\label{sec:discussion:guarantee}
\label{sec:discussion:judge}
\label{sec:discussion:evaluation}
\label{sec:discussion:system}
\label{sec:discussion:resources}
\label{sec:discussion:future}
\label{sec:exp:gate}
\label{sec:exp:ablation}
\label{sec:exp:scale}

We introduced \Fyan{}, a human--AI harness for document-level mathematical
formalization. Its design separates proof capability from semantic
faithfulness and formal correctness: models and human experts propose and
guide, while deterministic gates decide what enters the library. With the same
model, \Fyan{} proves \ftfpFyan{} of \Ntcs{} FormalTCS theorems under a strict
Lean check against \ftfpCtl{} for a general agent harness, and raises the
natural-language proof score from \tnpCtl{} to \tnpFyan{}. Its semantic audit
accepts only statements whose agreement with the source is evident, grades
every other difference, and names the object, hypothesis, claim, or scope at
fault; on ConsistencyCheck it catches more inconsistent statements than a
direct judge under both the original and the verified labels. \Fyan{} also
built ODENumLib, \odeDone{} dependency-linked theorem units of numerical
analysis.

For a library that later proofs import, reliability comes first. A false alarm
costs one more autoformalization round, whereas an accepted misformalization
silently undermines every proof built on it. Fixing one standard of semantic
equivalence in advance, and leaving every non-evident change to a
Lean-checked bridge, makes acceptance reproducible and as certain as the
evidence allows.

Our evaluation uses one model and one trajectory per item, ODENumLib covers a
single domain, and the audit costs more per statement than a single judgment.
Future work will extend the evaluation to more models and domains, ablate
\Curator{} and \Orchestrator{} supervision, and replace more of the Judge's
conversions with mechanically checked rules.

\clearpage
\subsection*{AI use statement}

Generative AI tools are an integral part of the \Fyan{} system studied in this
work and are used in the reported experiments for mathematical reasoning,
semantic analysis, and Lean proof construction. During the research process,
the authors also used generative AI tools to assist with software implementation
and debugging, literature
summarization, and drafting and editing the manuscript. AI-assisted mathematical
arguments, code, and manuscript content were reviewed by
the authors. The
authors take responsibility for the final content, claims, and artifacts
reported in this work.

\subsection*{Ethics statement}
This work does not involve human subjects or private or sensitive personal
data. \Fyan{} uses autonomous agents that may access executable tools, files,
and external information sources; such capabilities can introduce security and
resource-use risks if deployed without appropriate isolation. Our implementation
therefore scopes tool permissions, separates model-generated actions from
deterministic acceptance decisions, and is intended to run in an isolated,
low-privilege environment. Public benchmark data are used subject to their
respective licenses, and the source material used in the document-level case
study is open.

\subsection*{Reproducibility statement}
We provide detailed protocols and evaluation settings in
Appendix~\ref{app:evaluation}, additional semantic-audit analyses in
Appendix~\ref{app:audit-analysis}, and the complete ODENumLib case-study
inventory in Appendix~\ref{app:ode}. Appendix~\ref{app:audit-transfer} gives
the technical details of semantic validation and proof transfer,
Appendix~\ref{app:system} describes the system implementation and trust
boundaries, and Appendix~\ref{app:reproducibility} records the principal
configuration, prompt construction, and independent verification procedures.
The accompanying repository contains the implementation, prompts,
configurations, and supporting artifacts needed to reproduce the reported
experiments: \url{https://anonymous.4open.science/r/fyan-release-78C6}.

\subsection*{Acknowledgments}
This work was supported by National Key Research and Development Program
of China (No.~2022YFA1008200) and Shanghai Institute for Mathematics and
Interdisciplinary Sciences (SIMIS) under Grant SIMIS-ID-2025-ST. The
authors also acknowledge the assistance of GPT in preparing and revising the
manuscript.

\bibliography{references}
\bibliographystyle{plainnat}

\appendix

\section{Related Work}
\label{sec:related}

\paragraph{Mathematical reasoning and formalization.}
Language models assist mathematics through both informal reasoning and formal
proof construction. Informal systems generate, decompose, and critique
natural-language arguments, including multi-agent proof workflows
\citep{an2026qed,yang2025formalreasoning}. Formal systems translate statements
into proof-assistant languages~\citep{wu2022autoformalization,
azerbayev2023proofnet}, generate proofs or tactics
\citep{yang2023leandojo,xin2025deepseekproverv15,lin2025goedelproverv2}, and
combine informal sketches with formal search~\citep{jiang2023dsp,
ren2025deepseekproverv2}. Compiler-in-the-loop agents add diagnosis and repair
\citep{thakur2023copra,first2023baldur,ospanov2025apollo,
liu2026numinaleanagent}. Yet a checkable proof establishes only the supplied
formal statement, not that the statement preserves its informal source.

\paragraph{Semantic faithfulness.}
This distinction motivates work on semantic faithfulness.
Evaluations include expert assessment and curated informal--formal
pairs~\citep{wu2022autoformalization,azerbayev2023proofnet}, bidirectional
provability (BEq/BEq$^+$~\citep{liu2025beq}), learned or prompted alignment
scorers~\citep{lu2025formalalign,zhang2026beyondcompilation}, structural
distance~\citep{liu2025gted,liu2025assess}, expert-authored shadow statements
\citep{han2026shadowbench}, and deterministic checks for vacuity and benchmark
defects~\citep{ammanamanchi2026faults}. These signals are complementary:
expert evaluation is costly, structural similarity need not imply semantic
equivalence, and bidirectional provability of closed propositions need not
imply semantic faithfulness. Structured evaluators and
checklists decompose a global verdict into explicit criteria
\citep{liu2023geval,kim2024prometheus,viswanathan2025checklists,
gunjal2025rubricsasrewards}, but still require controls for omitted
comparisons and unsupported judgments.

\paragraph{Agents and document-level harnesses.}
Agentic proving and software-engineering systems combine decomposition, tool
use, feedback, repair, and reusable experience
\citep{thakur2023copra,wang2024legoprover,kumarappan2025leanagent,
yang2024sweagent,wang2025openhands}. Beyond single theorems,
\Archon~\citep{frenzymath2026archon,ju2026conjecture} coordinates proof work
across dependent declarations, while related efforts target libraries,
textbooks, and papers~\citep{mathinc2025gauss,zhang2026leanmarathon,
gloeckle2026textbook,rammal2026autoformbot,wang2026m2f,
cabral2026proofflow}. Blueprint-style workflows similarly support large human
formalization projects
\citep{massot_leanblueprint,scholze2022lte,tao2023pfrblueprint,
bolan2025equational}. These settings link statement assessment with artifact
versions, scheduling, and completion criteria. \Fyan\ integrates these ideas
through semantic auditing, checked transfer obligations, and deterministic
acceptance conditions around existing reasoning and proving systems.

\section{Experimental Details}
\label{app:evaluation}
\label{app:fts}

This appendix provides the experimental protocols and outcome accounting for
Section~\ref{sec:exp}. Additional analyses of semantic-audit behavior are
reported in Appendix~\ref{app:audit-analysis}.

\subsection{FormalTCS}
\label{app:fts-config}

FormalTCS evaluates three stages of formalization independently: NT2FT maps a
natural-language theorem to a Lean statement, T2NP generates a natural-language
proof from a theorem and its Lean statement, and FT2FP constructs a Lean proof
from a fixed statement. Because the benchmark provides annotated inputs
separately for each task, errors do not propagate across stages.

We compare \Fyan{} with a general-purpose agent harness using the same model,
Lean environment, released graders, and one candidate per item.
Table~\ref{tab:fts-config} summarizes the experimental configuration.

\begin{table}[h]
\centering
\normalsize
\caption{FormalTCS experimental configuration. Shared settings apply to both
the control and \Fyan{}.}
\label{tab:fts-config}
\setlength{\tabcolsep}{4pt}
\begin{tabular}{@{}p{2.2cm}p{10.9cm}@{}}
\toprule
Setting & Configuration \\
\midrule
Model
& DeepSeek-V4.1-Flash (\code{deepseek-flash}) \\

Sampling
& T2NP: $T=0$, top-$p=1$; NT2FT and FT2FP:
  $T=0.6$, top-$p=0.9$; one candidate per item, submitted once to the
  grader; no resampling \\

Lean
& Lean 4.32.2; Mathlib revision \code{905b958}; prebuilt cache unchanged \\

Graders
& Released benchmark graders: BEq$^+$ for NT2FT, the FormalTCS
  Qwen3.8-Max rubric for T2NP, and strict Lean verification for FT2FP \\

Control
& General-purpose agent using the benchmark prompts verbatim \\

\Fyan{}
& \Reasoner{} for T2NP and the \Formalizer{} for NT2FT/FT2FP,
  with the benchmark brief and stage preamble \\

Reference access
& Benchmark references are unavailable to both systems; reference files and
  credentials are hidden from the \Fyan{} agent by a \code{bwrap} sandbox \\
\bottomrule
\end{tabular}
\end{table}

The benchmark's default eight samples are reduced to one in both arms.

\paragraph{Outcome accounting.}
For NT2FT, the control produces 137
\code{not\_proven} outcomes, five \code{equivalent} outcomes, and one ill-typed
truncated file.

\subsection{ProofNet\#}
\label{app:proofnet}

ProofNet\# contains \pnsAll{} statements across its validation and test splits.
Under Lean 4.32.2, \pnsElig{} elaborate in the pinned environment and are
included in the evaluation. We run only the \Formalizer{} proof engine used
for FormalTCS, without the \Reasoner{} and \Planner{} stages, with one
trajectory and at most five prover iterations per theorem.

A theorem counts as solved only if it compiles, preserves the supplied
declaration, contains neither \code{sorry} nor a new axiom, and passes a
transitive axiom audit. The engine solves \pnsPass{} of the \pnsElig{} statements (\pnsPct\%). The remaining
outcomes are 109 proof failures, seven changed signatures, two compile failures,
and nine infrastructure failures; every non-pass is counted as a failure.

Published ProofNet results use a different Lean port and substantially larger
Pass@$N$ budgets, so we report this experiment only as an absolute
cross-corpus evaluation and do not use published results as a numerical
baseline.

\subsection{ConsistencyCheck}
\label{app:cc}

ConsistencyCheck contains 859 informal--formal pairs: 488 from miniF2F and 371
from ProofNet. Each statement is placed under its corresponding benchmark
header in Lean 4.32.2 with Mathlib revision \code{905b958}. We update retired
finite-sum notation in 66 statements, applying the same rewritten statement to
both evaluation methods. Ten items fail to elaborate because of library drift,
and 14 ProofNet entries are definitions rather than propositions, leaving 835
eligible items.

The reported evaluation uses \ccTN{} items sampled with a fixed seed, balanced
between 150 expert-labelled consistent and 150 inconsistent pairs. The
\Fyan{} judge runs the same source extraction, Lean-statement extraction,
local alignment, and deterministic validation used by \Fyan{}, with grades~2
and~3 treated as consistent. All other outcomes, including a missing valid
verdict, are treated as inconsistent. The direct baseline receives the same
source and Lean statement in a single model call, with the instruction to decide
whether the Lean statement faithfully formalizes the text: ``the same objects
and domains, the same hypotheses and the same conclusion, with nothing
missing, added, weakened or strengthened.'' Both methods use
DeepSeek-V4.1-Flash at low reasoning effort, temperature 0.6, and top-$p$ 0.9.

\begin{table}[h]
\centering
\normalsize
\caption{ConsistencyCheck confusion counts and aggregate metrics under the
original and the verified labels (Section~\ref{sec:analysis}).
``Inconsistent'' is the positive class. Invalid or missing verdicts of the
\Fyan{} judge are treated as predictions of inconsistency and are already included
in TP or FP; ``No verdict'' reports this diagnostic subset rather than an
additional outcome category. Brackets give 95\% bootstrap intervals.}
\label{tab:cc-counts}
\setlength{\tabcolsep}{4pt}
\begin{tabular}{@{}lrrrrrcc@{}}
\toprule
Method & TP & FN & TN & FP & No verdict & Accuracy & F1 \\
\midrule
\multicolumn{8}{@{}l}{\emph{Original labels}} \\
Direct judge
& 123 & \ccTDirFN & 142 & \ccTDirFP & 0
& \ccTDirAcc{} [\ccTDirAccCI] & \ccTDirFone \\
\Fyan{} judge
& 131 & \ccTAuditFN & 111 & \ccTAuditFP & \ccTNoVerdict
& \ccTAuditAcc{} [\ccTAuditAccCI] & \ccTAuditFone \\
\midrule
\multicolumn{8}{@{}l}{\emph{Verified labels}} \\
Direct judge
& \ccAdjDirTP & \ccAdjDirFN & \ccAdjDirTN & \ccAdjDirFP & 0
& \ccAdjDirAcc{} [\ccAdjDirAccCI] & \ccAdjDirFone \\
\Fyan{} judge
& \ccAdjAuditTP & \ccAdjAuditFN & \ccAdjAuditTN & \ccAdjAuditFP & \ccTNoVerdict
& \ccAdjAuditAcc{} [\ccAdjAuditAccCI] & \ccAdjAuditFone \\
\bottomrule
\end{tabular}
\end{table}

Under the original labels, the \Fyan{} judge is correct alone on
\ccTOnlyAudit{} items, including \ccTStruct{} structural errors, and the direct
judge on \ccTOnlyDir{}; under the verified labels, on \ccAdjMcAudit{} and
\ccAdjMcDir{} items respectively. The
main text reports recall first because the audit is used as a gate, where a
missed inconsistency propagates to downstream units while a false alarm
triggers another repair round; \ccTPolicy{} of the audit's \ccTAuditFP{} false
alarms are grade-1 reports (Appendix~\ref{app:audit-errors}). The
audit is also more expensive: its estimated model cost is about
\ccTAuditCost{} CNY per item, compared with \ccTDirCost{} CNY for a direct
judgment. Further analysis
of these disagreements, grade~1 cases, and representative failure modes is
given in Appendix~\ref{app:audit-analysis}.

\section{Additional Semantic-Audit Analysis}
\label{app:audit-analysis}

This appendix provides additional analyses supporting
Section~\ref{sec:analysis}. We focus on three issues: the limitations of
reference-based equivalence, the main sources of disagreement in the
ConsistencyCheck evaluation, and representative semantic mismatches that can
survive compilation or holistic judgment.

\subsection{Limits of Reference-Based Equivalence}
\label{app:beq}

We first examine the behavior of the released BEq$^+$ grader used for the
FormalTCS NT2FT task. A reference file graded against itself receives
\code{equivalent}, whereas renaming a separately declared structure in an
otherwise identical copy can produce \code{not\_proven}. The grader elaborates
the reference and candidate in separate namespaces and shares only their longest
verbatim common prefix, so independently declared structures, inductives, and
classes need not be identified automatically.

As an additional calibration, each of 21 references and 21 candidates graded
against itself receives \code{equivalent}. However, \genRefs{} of the
\Ntcs{} FormalTCS reference preambles introduce a \code{structure},
\code{inductive}, or \code{class}, and all \ntftCtl{} BEq$^+$ successes of the
control harness on NT2FT lie among the other \defRefs{} items; under random placement this has
probability $\binom{\defRefs}{\ntftCtl}/\binom{\Ntcs}{\ntftCtl}=\genFisher$.
Table~\ref{tab:beqcases} shows three
representative cases in which a semantically reasonable candidate is not
certified as equivalent to the benchmark reference.

\begin{table}[h]
\centering
\normalsize
\caption{Representative FormalTCS candidates reported as
\code{not\_proven} by BEq$^+$. Lean code is abridged.}
\label{tab:beqcases}
\setlength{\tabcolsep}{3pt}
\begin{tabular}{@{}
  >{\raggedright\arraybackslash}p{2.0cm}
  >{\raggedright\arraybackslash}p{3.75cm}
  >{\raggedright\arraybackslash}p{3.65cm}
  >{\raggedright\arraybackslash}p{3.8cm}@{}}
\toprule
Claim & Reference & Candidate & Reason for failure \\
\midrule
Expected number of cycles
 & average using Lean's total division
 & explicit zero branch for an empty family
 & equivalence requires a case split outside the tactic ladder \\

ReLU networks
 & inductive representation over affine maps
 & inductive representation over matrices and biases
 & independently declared representations cannot be identified \\

Best loss in a hypothesis class
 & bounded infimum whose empty fibers contribute zero
 & infimum over the image of the class
 & the reference does not encode the source infimum faithfully \\
\bottomrule
\end{tabular}
\end{table}

For the first case, a separate Lean proof establishes equivalence by splitting
on whether the averaging family is empty. In the third, a direct Lean check
shows that the bounded infimum in the reference may receive a zero contribution
from elements outside the intended class, whereas the candidate takes the
infimum over the class itself. These examples illustrate two distinct
limitations: failed proof search need not indicate semantic mismatch, and a
benchmark reference need not be a perfect semantic ground truth. We therefore
use BEq$^+$ as a complementary benchmark signal rather than as the sole
criterion for semantic faithfulness.

\subsection{ConsistencyCheck Error Analysis}
\label{app:audit-errors}

The \Fyan{} judge and the direct judge make different types of errors.
Under the original labels (Appendix~\ref{app:cc}), the audit reduces false
negatives (missed inconsistencies) but produces more false positives (false
alarms): \ccTAuditFP{}, of which \ccTPolicy{} at grade~1, \ccTAuditFPgZero{}
at grade~0, and \ccTAuditFPnv{} without a valid record. Re-verification
(below) finds that \ccVerFPNon{} of them are not the text's
statement---\ccVerGOneNon{} of the grade-1 reports, \ccVerGZeroNon{} of the
grade-0 rejections, and \ccVerNvNon{} of the runs without a record---and that
\ccVerFPSame{} are false alarms of the audit itself.

Grade~1 is where a binary label must make a policy choice. Among the
\ccTGradeOne{} grade-1 records, the original labels call \ccTPolicy{}
consistent and \ccTGradeOneInc{} inconsistent. Of the \ccTPolicy{},
\ccVerGOneArg{} are equivalent to the text only by an argument---for example
a codomain widened from a subspace to the whole space, or a binder domain
enlarged beyond what the text allows and restored only through the other
hypotheses---and \ccVerGOneOther{} rest on a reading that the text does not
state or differ from it. The original labels accept these statements, whereas
the audit, which certifies alignment, records them as substantive changes.
Lowering the acceptance threshold to~1 would pass them, but also the
\ccTGradeOneInc{} statements that the original labels call inconsistent: under
those labels, recall would fall from \ccTAuditR{} to \ccTThrOneR{}
(\ccTThrOneTP/150) while precision would rise to \ccTThrOneP{}. \Fyan{}
therefore keeps grade~2 as the direct-acceptance threshold and admits a
grade-1 statement only through a Lean-checked bridge.

The \ccVerFPSame{} false alarms that remain under the verified
labels---\ccVerGOneSame{} at grade~1, \ccVerGZeroSame{} at grade~0, and
\ccVerNvSame{} without a valid record---are attributable to relation-direction
errors, source parsing, slot pairing, or verdict production. These are
implementation failures rather than semantic disagreements about the label,
and they are the target of future improvements to the audit. All
\ccAdjAuditFN{} inconsistencies that the audit misses under the verified labels
are grade-2 acceptances, in which the Judge recorded each difference as one of
the listed conversions.

\paragraph{Re-verification of the labels.}
All \ccAdjN{} statements were re-verified. Each was first screened by Claude
Opus 5.5, which saw only the source text and the Lean statement, without the
original label or either method's prediction, and was assigned to one of four
classes: the same statement up to conversions a reader sees at a glance;
equivalent only by an argument; dependent on a reading or an assumption that
the text does not state; or not equivalent. The final class of every
statement was confirmed by the authors. Of the 150 statements that the
original labels call consistent, \ccAdjSame{} are the same as the text,
\ccAdjEqNotSame{} are equivalent only by an argument, \ccAdjUnstated{} depend
on an unstated reading or assumption (among them the finite dimension that the
linear-algebra exercises \code{exercise\_7\_9} and \code{exercise\_7\_11} take
from their textbook), and \ccAdjNotEq{} are not equivalent; none of the 150
statements labelled inconsistent is the same as its text. Relabelling the
\ccAdjRelabel{} statements outside the first class as inconsistent leaves
\ccAdjInc{} inconsistent and \ccAdjCon{} consistent statements. Under these
labels the \Fyan{} judge has TP/FN/TN/FP
\ccAdjAuditTP/\ccAdjAuditFN/\ccAdjAuditTN/\ccAdjAuditFP{}, accuracy
\ccAdjAuditAcc{} [\ccAdjAuditAccCI], precision \ccAdjAuditP{}, recall
\ccAdjAuditR{} [\ccAdjAuditRCI], and F1 \ccAdjAuditFone{}; the direct judge
has \ccAdjDirTP/\ccAdjDirFN/\ccAdjDirTN/\ccAdjDirFP{}, accuracy
\ccAdjDirAcc{} [\ccAdjDirAccCI], precision \ccAdjDirP{}, recall \ccAdjDirR{}
[\ccAdjDirRCI], and F1 \ccAdjDirFone{} (brackets: 95\% bootstrap intervals).
The \Fyan{} judge is correct alone on \ccAdjMcAudit{} statements and the direct
judge on \ccAdjMcDir{} (exact McNemar $p=\ccAdjMcP$). These are the verified
rows of Table~\ref{tab:cc}.

\paragraph{Runs without a verdict.}
All \ccTNoVerdict{} runs without a verdict end the same way: the Judge fails
three consecutive times to produce a record that passes validation, and the
fail-closed rule records a rejection. \ccTNoVerdictInc{} of them fall on
expert-inconsistent items and \ccTNoVerdictCon{} on expert-consistent ones.
Excluding them, its
recall is \ccTExclR{} (124/143) and its precision \ccTExclP{}; counting all
\ccTNoVerdict{} as acceptances, as a fail-open gate would, gives recall
\ccTOpenR{} (\ccTOpenTP/150) against \ccTDirR{} (123/150) for the direct judge,
with accuracy \ccTOpenAcc{}.

Two of the \ccTNoVerdictCon{} expert-consistent items sit on the boundary that
the audit is designed to expose. In \code{675\_exercise\_7\_9}, the source
states that a normal operator on a complex inner-product space is self-adjoint
if and only if all its eigenvalues are real, and the Lean statement adds
finite-dimensionality. The source textbook assumes finite dimension
throughout, but the statement itself does not, and without that assumption the
reverse direction can fail; the expert label therefore rests on context outside
the statement. In \code{145\_amc12a\_2013\_p7}, the source constrains ten terms
$S_1,\dots,S_{10}$ by a recurrence from $n\ge 3$, whereas the Lean sequence
$s:\mathbb{N}\to\mathbb{R}$ satisfies the recurrence for every $n$, including
indices outside the source range. The two are equivalent, because any ten terms
satisfying the source recurrence extend to a sequence satisfying it everywhere,
but establishing this needs an extension argument across a change of indexing,
the renumbering case that the Judge prompt treats explicitly. In both cases the
validator refused a record it could not check rather than accept the statement
on the model's word; in the pipeline, such a statement returns for another
round instead of entering the library unverified.

This behavior explains why the audit uses more than a binary
consistent/inconsistent decision. Grades~2 and~3 indicate alignment up to
trivial equivalence, grade~1 records a substantive but potentially recoverable
change, and grade~0 indicates an unsupported mismatch. A grade-1 statement is
therefore not accepted merely because it appears mathematically reasonable; it
requires the proof-transfer mechanism described in Section~\ref{sec:transfer}.

The classes used in this appendix come from the same re-verification as the
verified labels (Appendix~\ref{app:audit-errors}): screening by Claude
Opus~5.5, blind to the labels and to both methods' predictions, followed by
author confirmation of every class.

\subsection{Representative Semantic Failure Cases}
\label{app:audit-cases}

\paragraph{A localized scope error.}
The arithmetic-series example, case~(a) of Table~\ref{tab:audit-cases},
illustrates a mismatch that survives both compilation and holistic judgment. The source
requires
\[
  \sum_{k=0}^{4}(a+kd)=70,
  \qquad
  \sum_{k=0}^{9}(a+kd)=210,
\]
whereas the candidate's misplaced delimiter yields expressions of the form
$5a+kd=70$ and $10a+kd=210$, with $k$ outside the summation. The candidate
therefore constrains different mathematical objects even though the overall
statement remains syntactically plausible. The direct judge accepts the
statement, while the \Fyan{} judge rejects the affected hypothesis rows and
localizes the missing occurrence of the common-difference term.

\paragraph{A definition-sensitive mismatch.}
A second example concerns Herstein's Exercise~4.3.25. The source refers to an
ideal of the ring of $2\times2$ real matrices, where the intended notion is
two-sided. In Lean, however, \code{Ideal} over a noncommutative semiring denotes
a left ideal. Under the source notion, the stated conclusion holds; under the
Lean notion, it does not.

The evaluation record already contained Mathlib evidence indicating that
\code{Ideal} is a left ideal, but the comparison did not force the record to
check whether this formal notion matched the source notion of ``ideal.'' The
resulting judgment therefore accepted a statement with a semantically different
mathematical object. This example illustrates an important limitation of
retrieval alone: relevant evidence may be available without affecting the
verdict unless the audit explicitly requires the corresponding comparison.
Definition-sensitive objects must therefore be treated as part of the semantic
alignment rather than as incidental implementation details.

\section{ODENumLib Case Study}
\label{app:ode}

ODENumLib is a dependency-linked Lean development generated from
numerical-analysis lecture notes. It was built before the semantic audit was
added to \Fyan{}, so its statements were accepted on compilation and the
completion gate alone. The project contains \odeUnits{} theorem
units covering discrete Gr\"onwall inequalities, Taylor and interpolation
remainders, one- and multistep convergence, and stability of standard ODE
solvers. Table~\ref{tab:ode} gives the complete unit-level inventory.

The ``Statement'' column compares the final main theorem with the retained
pre-proof scaffold: $=$ denotes an identical statement, $\approx$ a change in
syntax or notation only, $\Delta$ a substantive mathematical change, and --
indicates that no pre-proof scaffold was retained. ``Current'' reports replay
under Lean 4.32.2 and Mathlib revision \code{905b958}.

\begin{table}[h]
\centering
\normalsize
\caption{ODENumLib unit inventory and replay status under the pinned
Lean/Mathlib environment.}
\label{tab:ode}
\setlength{\tabcolsep}{2pt}
\renewcommand{\arraystretch}{0.95}
\begin{tabular}{@{}
  >{\raggedright\arraybackslash}p{4.1cm}
  rcc
  >{\raggedright\arraybackslash}p{5.9cm}@{}}
\toprule
Unit & Lines & Statement & Current & Main result \\
\midrule
exponential\_bounds
& 111 & $\approx$ & pass & product--exponential bound \\

taylor/remainder
& 186 & $\approx$ & pass & Taylor remainder \\

ode
& 62 & $\Delta$ & pass & local existence and uniqueness \\

gronwall/constant
& 162 & $=$ & pass & constant-coefficient discrete Gr\"onwall \\

gronwall/variable
& 223 & $=$ & pass & variable-coefficient discrete Gr\"onwall \\

taylor/lagrange
& 637 & $\Delta$ & port & interpolation error \\

method/consistency\_stability
& 198 & $\Delta$ & pass & forward-Euler step stability \\

gronwall/multistep
& 520 & $=$ & pass & multistep Gr\"onwall \\

gronwall/sum\_form
& 176 & $\approx$ & pass & sum-form Gr\"onwall \\

method/absolute\_stability
& 713 & $\Delta$ & pass & stability region and L-stability \\

multistep\_theory
& 1,439 & $=$ & port & moment conditions for consistency \\

convergence/single\_step
& 273 & $\Delta$ & pass & one-step convergence bound \\

expert
& 451 & $\Delta$ & pass & stiff-ODE example \\

convergence/dahlquist
& 652 & -- & pass & one direction of the convergence theorem \\

forward\_euler
& 393 & $=$ & pass & first-order global error \\

runge\_kutta\_2
& 261 & $\Delta$ & pass & RK2 consistency and convergence \\

runge\_kutta\_4
& 340 & $=$ & pass & fourth-order convergence; local defect assumed \\

backward\_euler
& 486 & $\Delta$ & pass & well-posedness and convergence \\

trapezoidal
& 521 & -- & port & well-posedness and convergence \\

theta\_method
& 1,551 & $\Delta$ & pass & error bounds for $\theta$-methods \\
\midrule
adams\_bashforth
& -- & -- & open & never built; propagator bound assumed \\

adams\_moulton
& -- & -- & open & two \code{sorry} terms in the main theorem \\

bdf
& -- & $=$ & open & six \code{sorry} terms \\
\bottomrule
\end{tabular}
\end{table}

\paragraph{Replay status.}
In the original environment, 20 of the 23 units were closed without
\code{sorry} or newly introduced axioms. Under the pinned environment,
17 replay directly. The three previously closed units
\code{taylor/lagrange}, \code{trapezoidal}, and \code{multistep\_theory}
require porting because of library-level type changes, while the remaining
three units are still open. We therefore distinguish original completion from
current-environment replay rather than treating them as the same notion of
success.

\paragraph{Statement changes.}
Among the 19 units with a retained pre-proof scaffold, nine undergo a
substantive change to the main statement. These changes are not uniformly
problematic: some repair an inadequate scaffold, while others make the final
theorem weaker or easier than the source-facing result. Such weakening can also
enter before proof search, as two further units illustrate. The scaffold of
\code{runge\_kutta\_4}, kept unchanged, already assumes the local truncation
estimate that the source result is intended to derive. In
\code{convergence/dahlquist}, for which no scaffold was retained, the final
theorem proves only one direction of the stated convergence result.

These examples explain why project completion cannot be assessed solely by
whether individual Lean files close. Once a theorem becomes a dependency, any
substantive change in its statement becomes part of the interface seen by
downstream units, motivating semantic review before reuse.

\section{Semantic Audit and Proof Transfer: Technical Details}
\label{app:audit-transfer}
\label{app:rubrics}
\label{app:transfer}
\label{sec:link4}
\label{sec:policy}
\label{sec:link23}

This appendix provides the technical details behind the semantic audit and
proof-transfer mechanism described in Sections~\ref{sec:link1}
and~\ref{sec:transfer}. The audit separates two roles. A language model
constructs localized evidence relating the source specification to the
candidate Lean statement, while deterministic code validates the structure,
coverage, logical direction, and internal consistency of that evidence.
When an accepted formalization differs substantively from its source-facing
interpretation, the resulting transfer obligation must ultimately be
discharged by Lean.

\subsection{Structured Evidence and Deterministic Validation}
\label{app:audit-record}
\label{sec:link4:method}
\label{sec:link4:verdict}

\paragraph{Audit input and semantic inventory.}
\iffalse
The Judge compares the source statement with the candidate theorem signature
in its actual Lean environment. Its input contains the source statement and
digest, the extracted theorem signature, and a proof-stripped Lean surface in
which relevant local definition bodies are retained. Candidate-authored
comments are removed because they constitute neither source evidence nor
instructions to the audit.
\fi
The \Fyan{} Judge compares the source statement with the candidate theorem signature
in its actual Lean environment. Its input contains the source text; the source
tree, extracted and independently reviewed before any Lean is shown, with
literal quotations; the Lean tree that the host reads from the elaborated
declaration, with the bodies of the definitions it reaches; and the slots cut
from both trees, with host-computed scope and polarity. Candidate-authored
comments never reach the \Fyan{} Judge, because they constitute neither source
evidence nor instructions to the audit.

Rather than producing a free-form verdict, the \Fyan{} Judge fills a structured
evidence record. Its principal fields are summarized in
Table~\ref{tab:judge-record}.

\begin{table}[t]
\centering
\normalsize
\caption{Principal fields of the semantic evidence record.}
\label{tab:judge-record}
\begin{tabular}{@{}>{\raggedright\arraybackslash}p{0.33\linewidth}p{0.63\linewidth}@{}}
\toprule
Field & Purpose \\
\midrule
%\code{objects}, \code{hypotheses}, \code{claims}
%& Semantic inventory of the candidate statement. \\
%
%\code{lean}, \code{source}
%& Lean expression and locatable source support for each item. \\
%
%\code{scope}, \code{polarity}
%& Logical position and available binders or premises. \\
%
%\code{relation}, \code{proof\_level}, \code{evidence}
%& Local source--Lean relation and the evidence supporting it. \\
%
%\code{affects}, \code{omissions}
%& Downstream effects of changed objects and source content missing from Lean. \\
\code{source\_graph}, \code{candidate}
& Reviewed source tree and the host-extracted Lean tree. \\

\code{source\_slots}, \code{candidate\_slots}
& Binder, hypothesis, and goal slots cut by the host from each tree. \\

\code{scope}, \code{polarity}
& Logical position of each slot and the binders or premises available to it,
computed by the host. \\

\code{relation}, \code{conversion}, \code{argument}, \code{difference},
\code{evidence}
& Local relation of a group of paired slots or used definitions; the named
conversion, argument, or concrete difference supporting it; and an optional
\code{repair}. \\

\code{affects}, \code{omissions}
& Downstream effects of changed objects, and unpaired slots, which the host
scores as omissions. \\

\code{checks}, \code{transfer}
& Required semantic checks and, when needed, object, hypothesis, and claim
transfer evidence. \\
\bottomrule
\end{tabular}
\end{table}

\iffalse
The inventory is Lean-first: every formal object, hypothesis, and claim must
be supported by the source. Since this pass cannot detect source content that
has disappeared entirely, the Judge also performs reverse source coverage and
records missing objects, hypotheses, or claims as explicit omissions.
Many-to-many correspondence is allowed.
\fi
The comparison is source-first and covers both sides. The source tree is fixed
and reviewed before any Lean is shown, and the host cuts both trees into slots,
so a slot on either side that no row pairs remains visible and is scored by the
host as an omission under the rule given below. Many-to-many correspondence is
allowed.

Restrictions encoded through types, subtypes, instances, implicit parameters,
or definitions are treated as mathematical content. When an object or
representation changes, the record also identifies the downstream hypotheses
and claims whose meaning depends on that change.

\paragraph{Local semantic relations.}
The relation assigned to a local correspondence records what conversion is
required between the source and Lean item, rather than whether the two complete
theorems happen to share a truth value. Table~\ref{tab:audit-relations}
gives the relation vocabulary together with its required evidence.

\begin{table}[t]
\centering
\normalsize
\caption{Local semantic relations and their evidence requirements.}
\label{tab:audit-relations}
\begin{tabular}{@{}
  >{\raggedright\arraybackslash}p{0.28\linewidth}
  >{\raggedright\arraybackslash}p{0.23\linewidth}
  >{\raggedright\arraybackslash}p{0.42\linewidth}@{}}
\toprule
Relation & Evidence level & Interpretation \\
\midrule
\code{same}
& none
& Same mathematical content and structure, up to consistent renaming. \\

\code{trivial\_equivalent}
& immediate
& Immediate equivalence after notation or checked definitions are expanded. \\

\code{mutated\_equivalent}
& bounded local proof
& Substantive reformulation with evidence in both directions. \\

\code{mutated\_stronger}
& bounded local proof
& Candidate implies source. \\

\code{mutated\_weaker}
& bounded local proof
& Source implies candidate. \\

\code{mutated}
& bounded local evidence
& Substantive object or representation change whose effects are tracked
downstream. \\

\code{unaligned}
& unresolved or over budget
& Mismatch, unsupported comparison, or relation not established within the
allowed comparison budget. \\
\bottomrule
\end{tabular}
\end{table}

A bounded local proof is a short derivation from the source context, checked
definitions, and identified conversion lemmas, with the relevant side
conditions stated explicitly. Reproving the original theorem or invoking a
major result that bypasses the local comparison does not qualify.
Consequently, \code{unaligned} means that the required relation has not been
established; it does not by itself prove mathematical inequivalence.

Polarity is computed relative to the complete theorem. At positive polarity,
a one-way change must support candidate $\rightarrow$ source; at negative
polarity the direction reverses. A mixed-polarity occurrence requires both
directions unless a larger enclosing comparison directly establishes the
required transfer.

Omissions follow the same rule. A deleted conjunct is represented locally by
replacement with $\top$, while a deleted disjunct is represented by $\bot$.
The omitted item retains its original scope and polarity, so omission is not
handled by a separate heuristic such as ``dropping assumptions is safe.''

\paragraph{Deterministic validation.}
The model does not choose the final grade or pipeline transition. Before
grading, deterministic code checks the record schema, identifiers,
cross-references, scope, admissible relation values, evidence levels, and
required fields. \iffalse
It also requires explicit checks for definition lookup,
domain boundaries, operator--operand correspondence, scope and witness
dependencies, source coverage, and consistency between the recorded evidence
and transfer directions.
\fi
It also checks that the source tree passed its
independent review, that every slot lies in a row or is scored as an omission,
that each relation is admissible at the host-computed polarity of its slots,
that each \code{trivial\_equivalent} row names its conversions from a closed
list that includes standard definition unfolding, that substantive and one-way
rows state their argument and \code{unaligned} rows their concrete difference,
and that no placeholder stands in for evidence. The \Fyan{} Judge's audit of the
original text covers stated hypotheses, hidden extra conditions, conclusion
strength, realized notions, domains and types, and non-vacuity. Consistency
between the recorded evidence and transfer directions is also checked.

When a substantive change or omission is present, the record additionally
supplies three transfer obligations:
\code{object\_coverage}, \code{hypothesis\_transfer}, and
\code{claim\_transfer}. These correspond to the maps $e$, $a$, and $g$
introduced below. Transfer evidence cannot override a failed local comparison.

Malformed reports are distinct from semantic failures. If deterministic
validation rejects the structure of a record, the same \Fyan{} Judge may repair its
form against the unchanged source and candidate. This is form repair rather
than repeated semantic voting. Once a valid unfavorable record has been
produced, it is not resampled merely to obtain a different judgment.

Every validated item, omission, required check, and required transfer
contributes to the non-compensatory minimum aggregation described in
Section~\ref{sec:link1}, producing grades
$3$ (\code{aligned}), $2$ (\code{aligned\_trivial}),
$1$ (\code{modified}), and $0$ (\code{failed}).
These are ordered policy labels rather than probabilities or mathematical
truth values; in particular, unresolved or over-budget comparisons may also
receive grade~0.

\subsection{From Local Evidence to Global Proof Transfer}
\label{app:transfer:global}
\label{app:transfer:composition}

The semantic record is local, whereas a reusable theorem requires a
whole-statement guarantee. We now state a sufficient condition under which
scope-aware local certificates compose into a proof transformer.

Let
\[
S_0 :=
\forall x:X,\; A_0(x)\rightarrow G_0(x),
\qquad
S_1 :=
\forall y:Y,\; A_1(y)\rightarrow G_1(y),
\]
where $S_0$ is a source-facing Lean statement and $S_1$ is the audited
candidate.

Suppose there are terms
\[
e:X\rightarrow Y,
\qquad
a:\forall x:X,\;A_0(x)\rightarrow A_1(e(x)),
\]
and
\[
g:\forall x:X,\;
A_0(x)\rightarrow G_1(e(x))\rightarrow G_0(x).
\]
Here $e$ maps admissible source objects into the candidate domain, $a$
supplies the assumptions needed to invoke the candidate theorem, and $g$
recovers the source conclusion from the candidate conclusion.

\begin{proposition}[Proof transfer]
\label{thm:transfer}
Given $e$, $a$, and $g$ above, every proof $p:S_1$ induces
\begin{equation}
\label{eq:transfer}
\kappa(p)(x)(h)
=
g\,x\,h
\bigl(
p\,(e(x))\,(a\,x\,h)
\bigr),
\end{equation}
and hence a canonical proof transformer
$\kappa:S_1\rightarrow S_0$. Such transfers compose.
\end{proposition}

\begin{proof}
For $x:X$ and $h:A_0(x)$, the term $a\,x\,h$ has type
$A_1(e(x))$. Applying $p$ yields $G_1(e(x))$, and then
$g\,x\,h$ yields $G_0(x)$.

For composition, suppose $S_2\rightarrow S_1$ is witnessed by
$(e_{21},a_{21},g_{21})$ and $S_1\rightarrow S_0$ by
$(e_{10},a_{10},g_{10})$. Then
\[
e_{20}(x)=e_{21}(e_{10}(x)),
\]
\[
a_{20}(x,h)
=
a_{21}(e_{10}(x),a_{10}(x,h)),
\]
and
\[
g_{20}(x,h,k)
=
g_{10}\!\left(
x,h,
g_{21}(e_{10}(x),a_{10}(x,h),k)
\right)
\]
have the required types.
\end{proof}

The directions are forced by consequence: source assumptions must suffice to
apply the candidate theorem, whereas the candidate conclusion must suffice to
recover the source conclusion.

\paragraph{Signed local certificates.}
For corresponding source and candidate propositions $P$ and $Q$, define
\[
\mathrm{Dir}^{+}(P,Q):=Q\rightarrow P,
\qquad
\mathrm{Dir}^{-}(P,Q):=P\rightarrow Q.
\]
A signed comparison
$P\rightsquigarrow_{\sigma}Q$ therefore requests a certificate of type
$\mathrm{Dir}^{\sigma}(P,Q)$. The theorem root has positive polarity;
conjunction, disjunction, and quantifier bodies preserve polarity, while
negation and implication antecedents reverse it.

The comparison record forms a finite well-founded tree, or an acyclic shared
graph. Certified leaves form a non-overlapping comparison frontier; a leaf may
represent an atomic correspondence or a larger many-to-many group.
Unchanged regions contribute identity certificates.

Certificates are scope-aware. A certificate below a binder must hold uniformly
over that binder and may use only variables and premises available at its
declared location. It may not depend on sibling branches, future witnesses,
the candidate theorem being proved, or the source conclusion being recovered.
Dependencies among reused certificates must be acyclic.

For structured objects, aligned dependent telescopes and well-typed field
adapters construct the outer object map
\[
e:X\rightarrow Y.
\]

\begin{theorem}[Scope-preserving compositional transfer]
\label{thm:compositional-transfer}
Assume that:
\begin{enumerate}[leftmargin=*,nosep]
\item the object adapters are well typed and define $e:X\rightarrow Y$;
\item for every $x:X$, there is a complete scope-aware comparison
\[
C_A(x):
A_0(x)\rightsquigarrow_{-}A_1(e(x));
\]
\item for every $x:X$ and $h:A_0(x)$, there is a complete scope-aware
comparison
\[
C_G(x,h):
G_0(x)\rightsquigarrow_{+}G_1(e(x));
\]
\item all comparison leaves are covered and their certificates are valid in
their declared scopes.
\end{enumerate}
Then structural composition constructs
\[
a:\forall x:X,\;A_0(x)\rightarrow A_1(e(x))
\]
and
\[
g:\forall x:X,\;
A_0(x)\rightarrow G_1(e(x))\rightarrow G_0(x),
\]
and therefore the transformer
$\kappa:S_1\rightarrow S_0$ of Equation~(\ref{eq:transfer}).
\end{theorem}

\begin{proof}
The object adapters construct $e(x)$ in dependency order. The negative
comparison $C_A(x)$ compiles to
$A_0(x)\rightarrow A_1(e(x))$, yielding $a(x)$, whereas the positive
comparison $C_G(x,h)$ compiles to
$G_1(e(x))\rightarrow G_0(x)$, yielding $g(x,h)$. Substitution into
Equation~(\ref{eq:transfer}) gives the required transformer.
\end{proof}

\begin{lemma}[Signed certificate compilation]
\label{lem:signed-compilation}
Every well-formed comparison tree
$P\rightsquigarrow_{\sigma}Q$ compiles to a term of
$\mathrm{Dir}^{\sigma}(P,Q)$ in its declared context.
\end{lemma}

\begin{proof}
The proof is by structural induction. Certified leaves use their supplied
arrows and unchanged leaves use identity. Conjunction composes child maps
componentwise, disjunction by case analysis, and fixed-domain quantifiers
require the body certificate uniformly over the bound variable.

The essential reversal occurs at implication. At positive polarity, maps
\[
t_1:P_1\rightarrow Q_1,
\qquad
t_2:Q_2\rightarrow P_2
\]
give
\[
\lambda f\,p.\;
t_2\bigl(f(t_1(p))\bigr)
:
(Q_1\rightarrow Q_2)
\rightarrow
(P_1\rightarrow P_2).
\]
At negative polarity both arrows reverse. Because every recursive call stays
within the scope established by its parent, binder and witness dependencies
are preserved.
\end{proof}

The construction is sufficient rather than complete. A local comparison that
cannot be decomposed by these structural rules may still be recoverable through
an explicit certificate for a larger enclosing expression.

\subsection{Omissions and Changed Domains}
\label{app:transfer:rules}

\paragraph{Omissions.}
An omission remains explicit in the comparison tree. Deleting a conjunct is
modeled by replacing it with $\top$, whereas deleting a disjunct is modeled
by $\bot$. The elementary maps
\[
P\rightarrow\top,
\qquad
\bot\rightarrow P
\]
then determine admissibility through polarity. For example,
\[
(A\land H)\rightarrow G
\quad\leadsto\quad
A\rightarrow G
\]
is recoverable: removing $H$ weakens the candidate assumptions. In contrast,
\[
A\rightarrow(G\land H)
\quad\leadsto\quad
A\rightarrow G
\]
does not recover the original conclusion. Any normalization that removes the
introduced $\top$ or $\bot$ leaves the omitted item and its original scope
recorded in the audit.

\paragraph{Changed quantifier domains.}
Suppose the source body $P(u)$ ranges over $U$ and the candidate body $Q(v)$
over $V$. A domain change is admissible only when the transfer has an
appropriately directed witness adapter (Table~\ref{tab:domain-adapters}).

\begin{table}[t]
\centering
\normalsize
\caption{Direction-sensitive adapters for changed quantifier domains.}
\label{tab:domain-adapters}
\begin{tabular}{@{}ccll@{}}
\toprule
Binder & Sign & Adapter & Required body transfer \\
\midrule
$\forall$ & $+$ & $d:U\to V$
& $Q(d(u))\to P(u)$ \\

$\forall$ & $-$ & $d:V\to U$
& $P(d(v))\to Q(v)$ \\

$\exists$ & $+$ & $d:V\to U$
& $Q(v)\to P(d(v))$ \\

$\exists$ & $-$ & $d:U\to V$
& $P(u)\to Q(d(u))$ \\
\bottomrule
\end{tabular}
\end{table}

The directions depend on whether the transfer must construct or consume a
quantified witness. Adapters may depend on variables already in scope, but not
on future witnesses. Injectivity or surjectivity is not required by the
logical rule itself; whether the adapter preserves the intended mathematical
representation remains a separate semantic question.

\subsection{Bridge Obligations and Deterministic Completion}
\label{app:bridge}
\label{sec:link23:freeze}
\label{sec:policy:bridge}
\label{sec:link23:gate}
\label{sec:policy:gate}

The structured audit does not itself turn a non-identical formalization into a
source-faithful theorem. When the stage policy admits a non-identical reading,
the candidate root $S_1$ is frozen and the system records an explicit bridge
obligation to a source-facing statement $S_0$. The audit record stores the
semantic differences and transfer evidence motivating this obligation.

Freezing binds the accepted theorem to its normalized signature and to the
source and local definitions used during the audit. Subsequent changes to a
protected signature or one of its bound inputs make the corresponding
certification stale rather than silently transferring approval to the modified
statement.

The prover then proves $S_1$ without changing its frozen signature and, for
each bridge obligation, states a theorem of the form
\[
\code{theorem <root>\_faithful : } S_0
\]
and proves it from the frozen root. The symbolic composition described above
may guide this proof, but it does not discharge the obligation: only the
compiled Lean theorem establishes the formal implication
\[
S_1\rightarrow S_0.
\]
Before it is frozen, the bridge statement is itself subjected to semantic
review so that proving an implication to an incorrectly stated $S_0$ cannot
silently satisfy the workflow.

\paragraph{Completion gate.}
A model or prover cannot declare a theorem complete merely by emitting a
completion token. Final acceptance is determined by a deterministic gate that
makes no new semantic-model call. For each completion root, the gate requires:

\begin{enumerate}[leftmargin=*,nosep]
\item the target-scope \code{sorry} count is available and zero;
\item the project builds successfully under the configured Lean environment;
\item every frozen theorem signature and its bound audit inputs remain current;
\item every required bridge obligation has been discharged and the frozen
bridge declaration is still present;
\item every completion root carries the required semantic certification; and
\item a transitive axiom audit completes without \code{sorryAx} laundering or
an axiom outside the configured allow-list.
\end{enumerate}

These checks separate proof search from acceptance. Reasoners, retrieved
skills, \Orchestrator{} supervision, and human guidance may influence how a proof is found,
but they cannot bypass the frozen statement, bridge, build, or axiom
requirements used to decide completion.

\subsection{Executable Fragment and Guarantee Boundary}
\label{app:transfer:implementation}
\label{app:transfer:limits}
\label{sec:link4:tools}

The compositional checker implements a sufficient fragment of the construction
above. Its input consists of ordered source and candidate object telescopes,
explicit object adapters, paired formula trees, a non-overlapping comparison
frontier, local certificates, and optional certified normalizations.

The supported logical fragment contains $\top$, $\bot$, conjunction,
disjunction, implication, and fixed-domain universal and existential
quantifiers. The checker recomputes polarity, verifies coverage and scope, and
checks that each certificate has the direction required by its occurrence.
Groups appearing at both polarities require evidence in both directions.
Unsupported constructors, incomplete coverage, malformed scope, cyclic or
undeclared dependencies, incorrectly directed certificates, or a nonidentity
item labelled \code{same} cause refusal.

For omissions, the checker supports the explicit
$\land\top$, $\top\land$, $\lor\bot$, and $\bot\lor$ cases rather than
performing unrestricted logical simplification. General changed-domain
witnesses are not synthesized automatically; they must be provided explicitly
or justified by a larger enclosing certificate.

On success, the checker constructs symbolic recipes for $e$, $a$, and $g$,
together with the recomputed polarities and reindexings. These recipes are
conditional evidence, not Lean terms or kernel verdicts. In particular,
deterministic validation can verify that required evidence is present,
well-scoped, internally consistent, and used in the correct logical direction,
but it cannot establish the truth of an arbitrary free-form mathematical
explanation supplied by the model.

Three boundaries are especially important. First, polarity cannot be ignored:
a relation valid in a conclusion may require the opposite direction when it
appears in an implication antecedent. Second, local agreement does not justify
changes in witness dependency; for example,
\[
\forall x\,\exists y,\;y=x
\]
does not in general justify
\[
\exists y\,\forall x,\;y=x.
\]
Third, a map between domains is not sufficient by itself. Although the
inclusion $\mathbb{Z}\hookrightarrow\mathbb{R}$ may support an appropriate
universal restriction, it cannot transfer
\[
\exists r\in\mathbb{R},\;2r=1
\]
to
\[
\exists z\in\mathbb{Z},\;2z=1
\]
without a valid reverse witness construction.

Finally, proof transfer and semantic faithfulness remain distinct guarantees.
If $S_0$ faithfully expresses the source, $S_1$ has been proved, and Lean
checks a bridge $S_1\rightarrow S_0$, then a proof of $S_0$ is
machine-checked. The bridge does not, however, establish that $S_0$ is the
correct interpretation of the original natural-language source. That
responsibility remains with the semantic audit. Thus the structured evidence,
the compositional transfer record, and the Lean bridge certify different parts
of the overall formalization claim.

\section{System Implementation}
\label{app:system}
\label{sec:overview}
\label{sec:capability}
\label{sec:advisory}

This appendix describes the implementation details needed to understand how
\Fyan{} maintains provenance, dependency-aware execution, protected theorem
interfaces, and advice-only interaction across a document-level formalization
project. The emphasis is on the boundaries that affect the claims in the main
text; model prompts, numerical defaults, and reproduction commands are collected
separately in Appendix~\ref{app:reproducibility}.

\subsection{Project Construction, Provenance, and Scheduling}
\label{sec:overview:intake}
\label{sec:overview:stages}
\label{sec:overview:scheduling}
\label{sec:overview:workspace}
\label{app:specforge:phases}
\label{app:specforge:rubric}
\label{app:specforge:lint}
\label{app:specforge:honest}

\paragraph{\Specifier{} and theorem units.}
\Specifier{} converts a source document into self-contained theorem units
\code{problem.tex} and a document-level dependency DAG. Each extracted statement
is stored with source provenance, including its source file, page range, and
anchor. Persisted statement catalogues are never truncated; truncation is used
only when constructing bounded model contexts.

An architect model groups statements into theorem units and proposes
dependencies using the document structure together with a deterministic
reference graph extracted from explicit cross-references. The proposed DAG is
then canonicalized and checked independently. In particular, deterministic
validation requires every source statement to belong to exactly one unit,
checks that dependency targets exist, rejects cycles, and recomputes the
topological layering from the dependency relation rather than trusting a
model-generated ordering.

After the plan is accepted, units are generated layer by layer. Each
\code{problem.tex} contains the mathematical setup, dependencies, hypotheses,
and claim needed by downstream stages. Source provenance is inserted by code
rather than by the generator. A validation pass compares each generated unit
with its assigned source statements using a structured rubric for statement
coverage, hypotheses, added assumptions, conclusions, and dependencies.
Independent deterministic lints detect selected high-risk changes, including
lost biconditionals, altered strict inequalities, newly introduced asymptotic
terms, malformed provenance, and undeclared cross-unit dependencies.

Critical findings trigger bounded regeneration rather than silent acceptance.
If validation discovers an undeclared dependency on another theorem unit, the
corresponding DAG edge may be added automatically only after the updated graph
has again passed cycle and layering checks. Thus models propose the unit
decomposition, while coverage, graph well-formedness, provenance insertion,
and selected consistency properties are handled deterministically.

\paragraph{Persistent state and scheduling.}
A gateway is the sole writer of theorem-stage state. Stage advancement is
derived from persisted artifacts and exit status rather than from an agent's
self-report, so a failed stage remains active rather than being recorded as
complete.

A separate Python \Orchestrator{} schedules theorem units in topological order.
A unit becomes ready only after all of its declared dependencies are complete.
Starting a single theorem restricts execution to the transitive closure of its
dependencies, while project-wide execution dispatches independent ready units
up to a configured parallelism limit.

All theorem units belong to a single Lake project, but each proving run uses a
per-theorem workspace. Toolchain and package files are shared, while the current
project aggregate is copied into the workspace and sibling theorem units are
linked. An unfinished theorem may therefore be imported during its own proof
search without being exposed as a completed module to the project as a whole.
Its import into the product aggregate is activated only after completion.
Downstream units receive upstream results only from proved output modules, not
from draft statement scaffolds.

\subsection{Reasoning, Planning, and Lean Proof Construction}
\label{sec:qed}
\label{sec:qed:qed}
\label{sec:blueprint}
\label{sec:dag}
\label{app:archon:provenance}
\label{app:archon:loop}
\label{app:archon:tools}

\paragraph{\Reasoner{} reasoning.}
\Reasoner{} is based on the open-source \QED{} system~\citep{an2026qed}
for natural-language proof generation. We preserve its decomposition, proving, and
adversarial-review structure, while making two changes needed for
document-level formalization. First, the provenance block from
\code{problem.tex} is propagated into the resulting \code{proof.md}. Second,
a proof step may treat an already proved upstream theorem unit as a black box
rather than re-proving it.

The adapted stage also tests selected plan assertions by searching for
counterexamples and counting them only when they have been executed in a
sandbox. Repeated executed refutations force plan revision, while verified
proof fragments independent of the rejected step may be retained.

\paragraph{\Planner{} and the lemma-level DAG.}
\Planner{} turns \code{problem.tex} and the natural-language proof into
three main artifacts: a lemma-level dependency DAG, a human-readable
\code{Blueprint.md}, and a compiling \code{Statement.lean} scaffold.

The proof is first reviewed for logical gaps and translated into a lemma-level DAG at
a controlled granularity. The resulting graph is checked for dangling edges,
self-loops, target reachability, and source-step references. Candidate Mathlib
lemmas are collected from semantic search, LeanSearch, and
Loogle~\citep{gao2024leansearch,loogle}; every candidate name is then checked
against the live Lean environment before being passed back to the model.
Already proved upstream project lemmas are retrieved from their actual Lean
declarations.

The statement scaffold deliberately leaves theorem bodies as
\code{:= by sorry}. Its purpose is to freeze the intended interfaces before
proof search, not to obtain premature proofs by weakening difficult
conclusions. A bounded compile-and-quality pass checks that the required root
declaration exists, rejects vacuous targets such as \code{True}, and verifies
that every blueprint node has a corresponding tagged declaration.

Three dependency graphs therefore coexist at different levels. The
\emph{document-level DAG} schedules theorem units; the \emph{lemma-level DAG} organizes the
planned argument inside one unit and becomes read-only during proving; and the
\emph{proof DAG} may grow during proof construction as the prover introduces
helper declarations. These roles are distinct: document dependencies determine
what may be reused, while proof-local decomposition determines how a particular
theorem is attacked.

\paragraph{\Formalizer{} proving loop.}
\Formalizer{} started as a fork of \Archon{}
v0.2.0~\citep{frenzymath2026archon,frenzymath2026archonfirstproof,ju2026conjecture}
and has since been substantially rewritten: it runs on a Pi-based coding-agent
runtime~\citep{pi-agent}, and every consequential state transition is decided
outside the model. Little of the original remains beyond the plan--prove--review
structure of its proving loop.

A theorem progresses through autoformalization, proving, and polishing stages.
Within an iteration, the plan agent proposes proof objectives, prover agents work
on those objectives under a global concurrency limit, deterministic
post-processing validates touched Lean modules and synchronizes proof metadata,
and a reviewer summarizes progress for the next iteration.

Provers interact with Lean through goal-state inspection, diagnostics, hover and
declaration lookup, tactic trials, a code runner, axiom checking, and several
search interfaces. These tools may guide search but do not determine theorem
acceptance. Provers also do not directly control the project build or completion
state.

\subsection{Protected Artifacts and Deterministic Completion}
\label{app:archon:gate}
\label{app:archon:protect}
\label{app:archon:outside}

After a candidate root statement passes semantic review, \Fyan{} stores its
normalized signature together with bindings to the source-facing input and the
evidence used to certify it. Subsequent modification of the protected theorem
signature or a bound input invalidates that certification rather than silently
transferring approval to a changed theorem. Bridge obligations for admitted
non-identical statements are protected in the same way; their semantic and
logical role is described in Appendix~\ref{app:audit-transfer}.

An agent-generated \code{COMPLETE} token is therefore only a request. Final
completion is decided by deterministic code. For every completion root, the
system requires a zero target-scope \code{sorry} count, a successful current
project build, intact protected theorem signatures and semantic
certifications, discharged bridge obligations when required, and a transitive
axiom audit with no \code{sorryAx} or disallowed custom axiom.

Several other operations are likewise deliberately outside model control,
including objective parsing, \code{sorry} counting, axiom auditing, synchronization
of deterministic proof-status markers, and the final completion decision.
Missing or stale evidence fails closed.

The system distinguishes this acceptance evidence from human-facing summaries.
After a unit finishes, \Auditor{} combines deterministic scans of declarations,
\code{sorry}, \code{admit}, and \code{axiom} occurrences with a model-written
review. The prose report may aid inspection, but deterministic counts take
precedence and the report does not participate in theorem acceptance.

\subsection{Curated Reuse, Supervision, and Human Guidance}
\label{sec:curator}
\label{sec:curator:extraction}
\label{sec:curator:promotion}
\label{sec:curator:trust}
\label{sec:curator:retrieval}
\label{sec:supervision}
\label{sec:supervision:layers}
\label{sec:supervision:comments}
\label{sec:overview:advice}
\label{sec:capability:curator}
\label{sec:capability:scale}
\label{sec:capability:supervisor}

\paragraph{\Curator{}.}
\Curator{} extracts reusable proof experience from completed runs. It reads the
final Lean source together with the structured run log and records proof
patterns, tactic patterns, successful decomposition strategies, and informative
failure events. Related episodes are clustered and may be synthesized into
reusable prose skills.

Because these skills are model-generated, they are treated as search guidance
rather than proof evidence. Every lemma named by a candidate skill is checked
against the live Lean environment, and only fully verified, actionable skills
pass the strict retrieval gate. A model may subsequently narrow this admitted
set for a particular proof obligation, but it cannot introduce a skill that
failed deterministic admission.

The prover receives skill metadata and may open a skill explicitly when it is
relevant to the current obligation. Lemma names, tactics, imports, and side
conditions found in a skill remain suggestions to be checked against Lean.
Consequently, \Curator{} can change which proof search is attempted but cannot
change whether the resulting theorem satisfies the completion gate.

We do not attribute an independent empirical effect to \Curator{} in this paper.

\paragraph{\Orchestrator{} and human guidance.}
In addition to scheduling theorem units, the \Orchestrator{} independently
checks ongoing \Formalizer{} proofs automatically and regularly,
cross-checking the mathematical argument, specification, and proof plan. Its
supervision component uses the OpenClaw runtime~\citep{openclaw} and is
intentionally advice-only. Its model has no direct filesystem, shell, web, or executable-Lean
write capability; file mutation associated with supervision is performed by a
restricted sidecar.

The sidecar monitors coarse signs of stagnation, including an unchanged
positive \code{sorry} count, repeated blockers, and iterations with no prover
edits. Such events may trigger review or enable additional critic agents.
These event-triggered reviews supplement the regular automatic checks;
\Formalizer{} may also explicitly request a review.

\Orchestrator{} advice is inserted only as non-executable comments. Questions to a
human operator are routed through the Dashboard, and the returned answers enter
the proving context through the same comment-only channel. Human and \Orchestrator{}
guidance can therefore affect strategy and search, but neither can directly edit
a protected theorem statement, inject executable proof code, or bypass the
completion gate.

We do not evaluate \Orchestrator{} supervision as an independent capability intervention.
Its role here is architectural: it provides an external observation and
communication channel without becoming another authority over proof acceptance.

\subsection{Pipeline Trust Boundaries}
\label{app:tools}

Table~\ref{tab:trust-boundaries} summarizes the separation between components
that propose mathematical content or search actions and mechanisms that decide
whether artifacts may advance.

\begin{table}[t]
\centering
\normalsize
\caption{Principal trust boundaries in the \Fyan{} pipeline.}
\label{tab:trust-boundaries}
\setlength{\tabcolsep}{4pt}
\renewcommand{\arraystretch}{0.95}
\begin{tabular}{@{}
  >{\raggedright\arraybackslash}p{3.0cm}
  >{\raggedright\arraybackslash}p{4.8cm}
  >{\raggedright\arraybackslash}p{5.2cm}@{}}
\toprule
Component & May influence & Cannot determine directly \\
\midrule

\Specifier{} models
& theorem grouping, candidate dependencies, generated unit text
& graph validity, deterministic provenance, final downstream semantic
acceptance \\

\Reasoner{} / \Planner{}
& proof strategy, lemma decomposition, candidate scaffold
& final protected theorem interface or proof completion \\

\Formalizer{} agents
& proof plans, helper lemmas, executable proof search
& completion state, protected signatures, axiom acceptance \\

\Fyan{} Judge
& structured comparison evidence
& deterministic record validity or Lean proof validity \\

\Curator{}
& retrieved proof patterns and search suggestions
& whether a resulting proof is accepted \\

\Orchestrator{} / human
& strategic advice, questions, critic activation
& executable Lean, protected statements, or completion \\

Deterministic gates
& stage transitions and acceptance from persisted evidence
& mathematical interpretation beyond the evidence they validate \\
\bottomrule
\end{tabular}
\end{table}

\iffalse
The semantic Judge may use configured local and external research interfaces to
obtain evidence, while the proving agent may use Lean inspection and theorem
search tools.
\fi
The \Fyan{} Judge uses no research interfaces: each audit call receives only
its own input. The proving agent may use Lean inspection and theorem
search tools. Retrieved information is never itself an acceptance decision.
When required evidence, tool output, or protected state is unavailable, the
corresponding deterministic check fails closed.

This division is the central implementation principle of the harness:
model-based components expand the space of candidate formalizations, arguments,
and proof searches, whereas persisted evidence and deterministic checks control
what is allowed to become a reusable project artifact.

\section{Reproducibility and Configuration}
\label{app:reproducibility}
\label{app:config}
\label{app:prompts}
\label{app:reproduction}

This appendix records the pipeline configuration and reproduction
information that is not already specified in the experimental protocols of
Appendix~\ref{app:evaluation}. Experimental model, sampling, benchmark, and
Lean settings are reported with the corresponding evaluations; here we focus
on the defaults that govern the \Fyan{} workflow, semantic audit, and
document-level execution.

\subsection{System Configuration}
\label{app:reproducibility:config}

All project-level settings may be overridden in
\code{.archon/config.json}. Table~\ref{tab:config-archon} summarizes the
principal \Formalizer{} and semantic-audit defaults used by the system.

\begin{table}[t]
\centering
\normalsize
\caption{Principal \Formalizer{} and semantic-audit defaults.}
\label{tab:config-archon}
\setlength{\tabcolsep}{4pt}
\renewcommand{\arraystretch}{0.90}
\begin{tabular}{@{}
  >{\raggedright\arraybackslash}p{9.0cm}
  >{\raggedright\arraybackslash}p{4.0cm}@{}}
\toprule
Setting & Default \\
\midrule
Maximum proving iterations
& 50 \\

Maximum concurrent provers, including nested subagents
& 4 \\

Maximum dispatched objectives per iteration
& 10 \\

Block execution on incomplete dependencies
& true \\

Semantic protection
& enabled \\

Faithfulness threshold
& \code{trivial} \\

Non-identical-statement policy
& \code{bridge} \\

Maximum statement-repair attempts
& 3 \\

Maximum bridge-repair attempts
& 3 \\

%Structured-record form-repair attempts
%& 3 \\
%
%Research interfaces available to the semantic Judge
%& enabled \\
%
%Maximum Judge tool-calling rounds
%& 8 \\
%
%Maximum Lean surface supplied to the Judge
%& 60{,}000 characters \\
%
%Maximum Judge output
%& 32{,}768 tokens \\
%
%Judge timeout
%& 180 seconds \\
Structured-record form-repair attempts
& 3 per \Fyan{} Judge call; 5 for source-tree preparation \\

Source-tree repair rounds
& 2 \\

Re-pairing calls for unpaired slots
& 1 \\

Research interfaces available to the \Fyan{} Judge
& none \\

Maximum complete \Fyan{} Judge input
& 700{,}000 bytes, never truncated \\

Maximum \Fyan{} Judge output
& provider maximum for DeepSeek models (393{,}216 tokens); otherwise
32{,}768 tokens \\

\Fyan{} Judge timeout
& none (no local deadline) \\

Enabled proving subagents
& none by default \\

Maximum concurrently scheduled theorem units
& 3 \\

\Orchestrator{} retry budget
& 1 \\
\bottomrule
\end{tabular}
\end{table}

The semantic policy distinguishes the threshold used for direct acceptance
from the treatment of admitted non-identical statements. Under the default
\code{bridge} policy, a substantive deviation that is allowed to advance must
still satisfy the bridge mechanism described in
Appendix~\ref{app:audit-transfer}; configuration does not override protected
signatures or the deterministic completion gate.

Table~\ref{tab:config-other} gives the main defaults for source decomposition
and curated reuse.

\setlength{\textfloatsep}{11pt plus 1pt minus 1pt}
\begin{table}[t]
\centering
\normalsize
\caption{Principal \Specifier{} and \Curator{} defaults.}
\label{tab:config-other}
\setlength{\tabcolsep}{4pt}
\renewcommand{\arraystretch}{0.90}
\begin{tabular}{@{}
  >{\raggedright\arraybackslash}p{9.0cm}
  >{\raggedright\arraybackslash}p{4.0cm}@{}}
\toprule
Setting & Default \\
\midrule
\Specifier{} maximum units generated per call
& 6 \\

\Specifier{} maximum prior-context budget
& 6{,}000 tokens \\

\Specifier{} maximum verification revisions
& 2 \\

\Specifier{} maximum units repaired in validation
& 10 \\

Maximum unit TeX size admitted to validation
& 40{,}000 characters \\

\Curator{} promotion threshold
& 3 projects \\

\Curator{} minimum observations
& 3 \\

Single-project minimum skill size
& 180 characters / 2 lemmas \\

Maximum retrieved skills per run
& 20 \\

Maximum stuck-triggered retrievals per run
& 5 \\

Minimum review confidence before rejection
& 0.3 \\

Cluster merge thresholds
& Jaccard $\ge 0.7$ automatic;
  $[0.4,0.7)$ reviewed \\
\bottomrule
\end{tabular}
\end{table}

\setlength{\parskip}{4.5pt}
These defaults specify operational budgets rather than correctness
guarantees. Changing them may alter search effort or the frequency of review,
but final theorem acceptance remains subject to the protected-artifact,
build, bridge, and axiom checks described in
Appendix~\ref{app:system}.

\subsection{Prompts and Semantic-Audit Input}
\label{app:reproducibility:prompts}

\iffalse
The semantic Judge uses
\code{.archon-src/prompts/faithfulness-judge.md}; a project may override this
file at
\code{.archon/prompts/faithfulness-judge.md}.
The user message assembled for each audit contains:

\begin{enumerate}[leftmargin=*,nosep]
\item the source statement together with its SHA-256 digest;
\item the reviewed Lean theorem signature;
\item a proof-stripped Lean surface in which relevant definition bodies are
retained and candidate-authored comments are removed; and
\item an indication of which research interfaces are available.
\end{enumerate}

The prompt asks the model to produce the structured semantic evidence record
described in Appendix~\ref{app:audit-transfer}. Research results supplied
through local search, LeanSearch, Loogle, web search, or page retrieval are
treated as evidence for particular comparisons rather than as acceptance
decisions.\looseness=-1
\fi
The semantic audit uses the prompts in
\code{.archon-src/prompts/coverage/}. Each stage is a fresh model call that
receives only its own input: \code{source.md} extracts the source tree with
literal quotations; \code{source\_review.md} reviews it independently;
\code{source\_repair.md} and \code{source\_repair\_review.md} repair a
rejected tree and review the repair; \code{alignment.md} is the \Fyan{} Judge;
\code{pairing.md} is appended to the unchanged \Fyan{} Judge input for one re-pairing
call when slots remain unpaired; and \code{discrepancy.md} routes a rejection
to source-tree repair, candidate re-extraction, statement repair, or source
clarification. The user message assembled for each \Fyan{} Judge call contains:

\begin{enumerate}[leftmargin=*,nosep]
\item the source text;
\item the reviewed source tree;
\item the candidate Lean tree extracted from the elaborated declaration, with
the bodies of the local definitions it reaches; and
\item the host-computed source and candidate slots with their polarities, and
the polarity of each source definition.
\end{enumerate}

The prompt asks the model to produce the structured semantic evidence record
described in Appendix~\ref{app:audit-transfer}; its full text is reproduced
at the end of this subsection.\looseness=-1

If a generated record fails deterministic schema or consistency validation,
a bounded form-repair attempt receives the validator complaint and the rejected
record as untrusted input. These retries repair the structure of the record;
a valid unfavorable semantic judgment is not repeatedly sampled as a vote.\looseness=-1

The proving pipeline uses separate prompts for statement construction and
polishing, including \code{prover-autoformalize.md} and
\code{prover-polish.md}. \Specifier{} uses
\code{system\_validator.md} for model-assisted unit validation.
The exact prompt files used for a reported run should be retained together
with the corresponding run artifacts rather than reconstructed from their
names alone.\looseness=-1

\paragraph{Judge prompt.}
The complete system prompt of the \Fyan{} Judge (\code{alignment.md}) follows
verbatim. In it, \code{equivalent}, \code{stronger}, and \code{weaker} name
the relations \code{mutated\_equivalent}, \code{mutated\_stronger}, and
\code{mutated\_weaker} of Table~\ref{tab:audit-relations}.
\begin{lstlisting}[basicstyle=\ttfamily\fontsize{6.6}{7.9}\selectfont,breaklines=true,
  columns=fullflexible,frame=single,framerule=0.3pt,xleftmargin=2pt,
  xrightmargin=2pt]
Compare a Lean statement with the natural-language source text ("source"). You
receive the reviewed source tree ("source_graph"), the actual Lean candidate
extracted from its elaborated Expr ("candidate"), and host-computed slots. Neither
tree can be edited. The original text is authoritative for meaning. Do not test
whether the theorem is true. Compare what the two statements say.

Slots. The host strips the outer telescope of each statement into slots: "binder" (a
quantified variable and its domain), "hypothesis" (one conjunct of an assumption)
and "goal" (one conjunct of the conclusion); binders and hypotheses restrict the
claim (negative), goals assert it (positive). Curried and conjoined hypotheses give
the same slots; structure nested inside a slot is one unit. Each slot lists id,
kind, polarity and "text" (source) or "lean" (candidate); "definition_polarity"
gives each source definition the sign of its uses ("neutral" = used as data;
"unused" = the formula never uses it: it needs no group). Write one row per group of
corresponding slots (a group means the conjunction of its slots under the paired
binders). Pair binders, hypotheses and goals in separate rows, unless one side is a
single slot that the other unpacks. Leave a slot with no counterpart out of every
row; the host scores it (a dropped source hypothesis or an extra candidate
conclusion makes a stronger theorem; anything else missing or extra is a
difference). A candidate binder marked "context" needs a row only if it realizes a
source variable; pair every other binder, e.g. [Field K] bundled with its carrier K.
A source hypothesis the candidate's types enforce joins that binder's row (condition
moved into a type). A candidate binder naming a value the source quantifies inside a
hypothesis ((exists L, P L) -> Q is forall L, P L -> Q) joins that hypothesis row
(currying); one whose hypotheses realize a source definition (initial values, a
recurrence) joins the row of the source slots using it (inlining); it must occur in
a candidate hypothesis of its row and in no candidate hypothesis outside it. List a
"defined_by" binder (v = t) realizing a used source definition, without its
hypothesis (placed with it), in that group's "candidate_slots"; the group is at best
trivial_equivalent (inlining).

Relations, for a slot group or a definition group (P = source part, L = Lean part):
- same: identical meaning and representation, one slot per side of the same kind
  (bound-variable renaming and de Bruijn shifts from proof binders are allowed).
- trivial_equivalent: identical meaning. The representation changes only by
  conversions from this list, named in "conversion" (comma-separated): bundling,
  currying, binder order, argument order, condition moved into a type, Mathlib vs
  local definition, inlining, implicit context, renaming, reindexing, standard
  unfolding. A row that regroups slots names bundling, condition moved into a type
  or standard unfolding (or the currying or inlining above); splitting a
  conjunction or a chained inequality (0 < d <= n) into several slots is bundling.
  If the definition groups already record every change a slot row needs, the row
  names those conversions. Standard unfolding: a difference an undergraduate who
  knows the concepts sees at a glance: elementary logic (not exists vs forall not,
  currying), a standard definition unfolded (units digit vs % 10, percent vs /100,
  degrees vs radians, abelian vs forall a b, a*b = b*a, a Mathlib predicate vs its
  defining formula), or a change of domain that an accompanying bound neutralizes
  (n : Int with n >= 9 vs n : Nat with 9 <= n; exists! k : Nat vs exists! k : Int
  when k > 0 is forced). Name it and give the one-line unfolding in "argument".
- equivalent: L and P are equivalent by a real argument beyond the listed
  conversions. Give it in "argument".
- stronger: L implies P, uniformly in the enclosing variables. For a binder the
  Lean domain is SMALLER (Nat for "integer"); a LARGER Lean domain is weaker (Nat
  for "positive integer" adds 0; Complex for a real; Int for Nat). Give "argument".
  A binder and the hypothesis that bounds it form one group: judge them together
  in one row (bundling, condition moved into a type or standard unfolding).
- weaker: P implies L, uniformly in the enclosing variables. Give "argument".
- unaligned: the meaning differs or cannot be established. Give "difference":
  what the source says versus what the candidate says, preferably an instance
  where they disagree (an index shift that changes which elements are covered,
  a flipped inequality, a different range or set, a missing case).
Choose the relation by what it takes to get from P to L, not by whether the theorem
survives. Any row may add "repair": the change that would make it same. Mixed rows
(joining negative and positive slots) and neutral definitions admit only same,
trivial_equivalent or equivalent. Be generous with genuine re-encodings; reserve
unaligned for changes in which situations the claim covers or what it says. Text
indexing a sequence from 1 vs Lean Nat -> A: hypothesis instances that only
constrain indices outside the source range (a recurrence at n = 0) are reindexing if
the goal never reads those indices and every source sequence has a value there
satisfying them (name it, e.g. a 0 = 0); otherwise they restrict.

Definitions. Each source definition the formula uses belongs to exactly one group
{"source_ids","candidate_names","members","parameters",relation,...}. Members assign
each source definition its candidate declaration roots (their union is
candidate_names; an empty list means inlined or missing, never same). Parameter
pairs map a source parameter to the FULL candidate declaration name and the index of
its lambda binder; use parameters=[] when none correspond (e.g. bundling). Read
every local candidate body: a definition means what its body computes, not what its
name says. A named notion the text does not expand is compared with what the text,
or the standard meaning of that name, says: its body must be that notion (same) or
it is unaligned. Realizing it concretely is never a conversion. Only if nothing
constrains the notion may you record the realization in "premises". Candidate names
may also be exact names from external_primitives: imported constants carry their
documented meaning (their bodies are not exported, which is not missing evidence).
An imported anchor is same only for a source symbol, alone or applied to its
definition's parameters in order, with no dependency on another definition; an
expanded source construction realized by an imported definition is
trivial_equivalent, "Mathlib vs local definition". List every local candidate
declaration except the target exactly once in "declarations" with its linked source
ids; empty links only for generated or inert machinery.

"premises": plain sentences you assumed but could not check (an imported constant's
meaning, the tree's reading of a recorded ambiguity). They never block. A premise
may never assume that a local definition means what the text says. Doubt is not a
finding: a believed difference is a row or a defect.

"meaning_audit" reads the ORIGINAL TEXT for what the rows cannot show, in this
order, each {"check","status":"ok"|"defect","evidence"}: hypotheses_present (each
condition the text states is in a source slot), no_extra_conditions (no candidate
condition hides in a definition, type or module that no row states),
conclusion_strength (each conclusion the text asserts is a source goal slot, with
its quantifier scope), notions_realized (each notion the text defines or constrains
is a definition with that body, not a free parameter, field, axiom, arbitrary choice
or trivially true body; name each; an unconstrained field of a locally declared
structure is such a defect: quote the text, name the field), domains_types (Nat
contains 0 and truncates n - 1, x / 0 and x % 0: if the source domain excludes 0,
compare the n = 0 case), non_vacuity (hypotheses satisfiable, conclusion not
trivial). A difference that a row or an unmatched slot records is not a defect. A
defect adds "source_quote" (verbatim text) and "lean", and states the difference.

A candidate field {"$ref":"E0001"} is an entry of the lossless expression_pool;
expand it recursively. Author no score or verdict; a valid unfavorable row is final.

Return JSON:
{"definitions":[{"source_ids":["D1"],"candidate_names":["Foo"],"relation":"same",
 "evidence":"...","members":[{"source_id":"D1","candidate_names":["Foo"]}],
 "parameters":[{"definition":"D1","source":"x","candidate_name":"Foo","candidate_index":0}]}],
 "declarations":[{"name":"Foo","source_ids":["D1"],"reason":"..."}],
 "slots":[{"source_slots":["s1"],"candidate_slots":["c1"],"relation":"same","evidence":"..."},
  {"source_slots":["s2"],"candidate_slots":["c2","c3"],"relation":"trivial_equivalent",
   "conversion":"condition moved into a type","evidence":"..."},
  {"source_slots":["s4"],"candidate_slots":["c5"],"relation":"weaker","argument":"...",
   "repair":"...","evidence":"..."}],
 "premises":[],
 "meaning_audit":[{"check":"hypotheses_present","status":"ok","evidence":"..."}, ...]}
\end{lstlisting}

\subsection{Reproduction Records and Independent Verification}
\label{app:reproducibility:verification}

A reproducible run requires more than the final generated theorem. For each
reported result, the reproduction record should bind the output to the
corresponding source corpus, configuration, composed prompts, model and
provider settings, random seed where applicable, Lean and Mathlib revisions,
validation code, aggregation code, retained run logs, and independent
verification output. Cryptographic hashes may be used to identify these
artifacts unambiguously.\looseness=-1

Proof claims are independently checked after generation. The verifier
recompiles each claimed theorem in the configured Lean environment, checks
that the protected declaration agrees with its frozen signature, rejects
\code{sorry} and \code{admit}, and audits the transitive axioms of every
completion root. When a bridge is required, the bridge theorem is included in
the same verification procedure.\looseness=-1

Semantic evidence and formal verification are retained separately. The former
records why the system considered a statement aligned with, or recoverable
from, its source; the latter establishes that the resulting Lean declarations
compile and depend only on permitted axioms. Reproducing only the final Lean
file therefore does not reproduce the full pipeline acceptance decision.
Resource accounting follows the same principle. If runtime usage is
unavailable for a run, that run is excluded from aggregate cost claims rather
than being assigned zero cost.

\end{document}